\documentclass{article} 

\usepackage[accepted]{icml2020}

\usepackage[hidelinks]{hyperref}
\usepackage{url}

\usepackage[utf8]{inputenc} 
\usepackage[T1]{fontenc}    
\usepackage{hyperref}       
\usepackage{url}            
\usepackage{booktabs}       
\usepackage{amsfonts}       
\usepackage{nicefrac}       
\usepackage{microtype}      

\usepackage{amsmath,amsfonts,bm}

\def\eqref#1{equation~\ref{#1}}

\def\1{\bm{1}}

\def\eps{{\epsilon}}

\DeclareMathAlphabet{\mathsfit}{\encodingdefault}{\sfdefault}{m}{sl}
\SetMathAlphabet{\mathsfit}{bold}{\encodingdefault}{\sfdefault}{bx}{n}

\DeclareMathOperator*{\argmin}{arg\,min}

\usepackage{soul}
\usepackage{graphicx}
\usepackage{comment}
\usepackage[flushleft]{threeparttable}

\usepackage[algo2e,ruled,vlined]{algorithm2e}

\SetCommentSty{mycommfont}
\usepackage{mathtools}
\usepackage{amsthm}
\usepackage{booktabs}
\usepackage{amsmath,amssymb}
\usepackage{amsfonts}
\usepackage{booktabs}
\usepackage{subcaption}
\usepackage{siunitx}
\usepackage[textsize=tiny,textwidth=8em,backgroundcolor=yellow]{todonotes} 
	
\usepackage{wrapfig}
\usepackage{colortbl}
\definecolor{Lightgray}{RGB}{235,235,235}

\usepackage{arydshln}
\usepackage{wrapfig}
\usepackage{tabularx}

\DeclarePairedDelimiterX{\infdivx}[2]{(}{)}{%
  #1\;\delimsize\|\;#2%
}

\newtheorem{theorem}{Theorem}
\newtheorem{lemma}[theorem]{Lemma}
\newtheorem*{lemma*}{Lemma}
\newtheorem{claim}[theorem]{Claim}

\newtheorem{definition}[theorem]{Definition}

\definecolor{c1}{RGB}{239,71,111}
\definecolor{c2}{RGB}{250,131,52}

\DeclareMathOperator*{\Exp}{\mathbb{E}}

\usepackage{enumitem}

\newenvironment{myenumerate}
{\begin{enumerate}[leftmargin=5.5mm, itemsep=1pt, topsep=2pt]}
	{\end{enumerate}}

\icmltitlerunning{Fundamental Tradeoffs between Invariance and Sensitivity to Adversarial Perturbations}

\begin{document}

\twocolumn[
\icmltitle{Fundamental Tradeoffs between Invariance and \\ Sensitivity to Adversarial Perturbations}

\icmlsetsymbol{equal}{*}

\begin{icmlauthorlist}
\icmlauthor{Florian Tramèr}{stanford}
\icmlauthor{Jens Behrmann}{bremen}
\icmlauthor{Nicholas Carlini}{google}
\icmlauthor{Nicolas Papernot}{google}
\icmlauthor{Jörn-Henrik Jacobsen}{vector}
\end{icmlauthorlist}

\icmlaffiliation{stanford}{Stanford University}
\icmlaffiliation{google}{Google Brain}
\icmlaffiliation{vector}{Vector Institute and University of Toronto}
\icmlaffiliation{bremen}{University of Bremen}

\icmlcorrespondingauthor{Florian Tramèr}{tramer@cs.stanford.edu}

\icmlkeywords{Machine Learning, ICML}

\vskip 0.3in
]

\printAffiliationsAndNotice{} 
\begin{abstract}

Adversarial examples are malicious inputs crafted to induce misclassification. Commonly studied \emph{sensitivity-based} adversarial examples introduce semantically-small changes to an input that result in a different model prediction. 
This paper studies a complementary failure mode, \emph{invariance-based} adversarial examples, that introduce minimal semantic changes that modify an input's true label yet preserve the model's prediction.
We demonstrate fundamental tradeoffs between these two types of adversarial examples. 
We show that defenses against sensitivity-based attacks 
actively harm a model's accuracy on invariance-based attacks, and that new approaches are needed to resist both attack types.
In particular, we break state-of-the-art adversarially-trained and \emph{certifiably-robust} models by generating small perturbations that the models are (provably) robust to, yet that change an input's class according to human labelers.
Finally, we formally show that the existence of excessively invariant classifiers arises from the presence of \emph{overly-robust} predictive features in standard datasets. 
\end{abstract}

\section{Introduction}

Research on adversarial examples~\citep{szegedy2013intriguing,biggio2013evasion} is motivated by a spectrum of questions, ranging from the security of models in adversarial settings~\citep{tramer2019ads}, to limitations of learned representations under natural distribution shift~\citep{GilmerMotivating2018}. 
The broadest accepted definition of an adversarial example is ``an input to a ML model that is intentionally designed by an attacker to fool the model into producing an incorrect output''~\citep{goodfellow2017attacking}.

Adversarial examples are commonly formalized as inputs obtained by adding some perturbation to test examples to change the model output.
We refer to this class of malicious inputs as \textit{sensitivity-based adversarial examples}. 

To enable concrete progress, the adversary's capabilities are typically constrained by bounding the size of the perturbation added to the original input. The goal of this constraint is to ensure that semantics of the input are left unaffected by the perturbation. In the computer vision domain, $\ell_p$ norms have grown to be a default metric to measure semantic similarity. This led to a series of proposals for increasing the robustness of models to sensitivity-based adversaries that operate within the constraints of an $\ell_p$ ball~\cite{madry2017towards, kolter2017provable,raghunathan2018certified}. 

In this paper, we show that optimizing a model's robustness to $\ell_p$-bounded perturbations
is not only insufficient to resist general adversarial examples, but also potentially \textit{harmful}. As $\ell_p$ distances are a crude approximation to the visual similarity in a given task, over-optimizing a model's robustness to $\ell_p$-bounded perturbations renders the model excessively invariant to real semantics of the underlying task.

Excessive invariance of a model causes vulnerability against \emph{invariance adversarial examples}~\citep{jacobsen2018excessive}. 
These are perturbations that change the human-assigned label of a given input
but keep the model prediction unchanged.
For example in Figure~\ref{fig:myfig} an MNIST image of a digit `3' is perturbed
to be an image of a `5' by changing only 20 pixels; models that are excessively invariant
do not change their decision and incorrectly label both images as a `3', despite the
fact that the oracle label has changed.

\begin{figure*}[t]
\centering
\includegraphics[width=.7\textwidth]{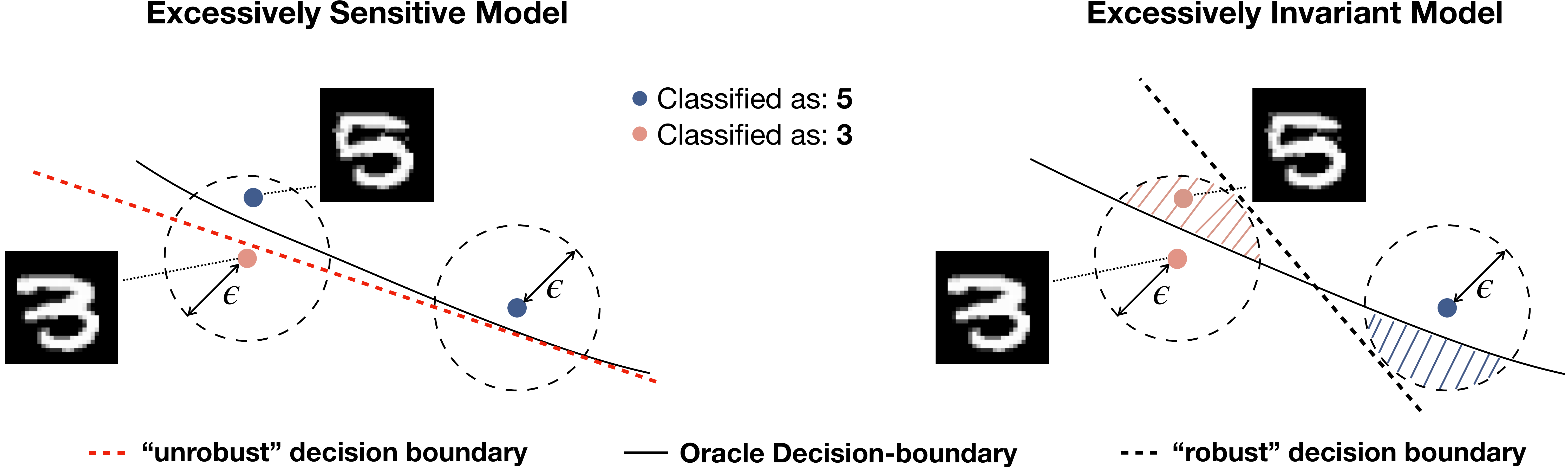}
\vspace{-0.75em}
\caption{
Decision boundaries near a real image of a digit `3' and an invariance-based adversarial example labeled as `5'. [Left]: Training a classifier without constraints may learn a decision boundary unrobust to sensitivity-based adversarial examples. [Right]: Enforcing robustness to norm-bounded perturbations introduces erroneous invariance (dashed regions in $\varepsilon$-spheres). We display real data here, the misclassified `5' is an image found by our attack which resides within a typically reported $\varepsilon$-region around the displayed `3' (in the $\ell_0$ norm). This excessive invariance of the robust model in task-relevant directions illustrates how robustness to sensitivity-based adversarial examples can result in new model vulnerabilities.}
\label{fig:myfig}
\vspace{-0.75em}
\end{figure*}

This paper exposes a fundamental tradeoff between sensitivity-based and invariance-based adversarial examples. 
We show that due to a misalignment between formal robustness notions (e.g., $\ell_p$-balls) and a task's perceptual metric, current defenses against adversarial examples cannot prevent both sensitivity-based and invariance-based attacks, and must trade-off robustness to each (see Figure~\ref{fig:alignment}). Worse, we find that increasing robustness to sensitivity-based attacks \textbf{decreases} a model's robustness to invariance-based attacks. 
We introduce new algorithms to craft
$\ell_p$-bounded invariance-based adversarial examples, and illustrate the above tradeoff on MNIST.\footnote{While MNIST can be a poor choice for studying adversarial examples, we chose it because it is the \textit{only} vision task for which models have been made robust in non-negligible $\ell_p$ norm balls. The fundamental tradeoff described in this paper will affect other vision tasks once we can train strongly robust models on them.}
We show that state-of-the-art robust models disagree with human labelers on many of our crafted invariance-based examples, and that the disagreement rate is higher the more robust a model is. We find that even models robust to very small perturbations (e.g., of $\ell_\infty$-norm below $\varepsilon=0.1$) have higher vulnerability to invariance attacks compared to undefended models.

We further break a \textit{provably-robust} defense~\citep{zhang2019towards} with our attack. This model is certified to have 87\% test-accuracy (with respect to the MNIST test-labels) under $\ell_\infty$-perturbations of radius $\varepsilon=0.4$. 
That is, for $87\%$ of test inputs $(x, y)$, the model is guaranteed to predict class $y$ for any perturbed input $x'$ that satisfies $\|x-x'\|_\infty \leq 0.4$.
Yet, 
on our invariance-based adversarial examples that satisfy this norm-bound, the model only agrees with human labelers in $60\%$ of the cases for an automated attack, and $12\%$ of the cases for manually-created examples---i.e., no better than chance. The reason is that we can find perturbed inputs $x'$ that humans no longer classify the same way as $x$.

Code to reproduce our attacks is available at {\footnotesize \url{https://github.com/ftramer/Excessive-Invariance}}.

Finally, we introduce a classification task where the tradeoff between sensitivity and invariance can be studied rigorously. We show that excessive sensitivity and invariance can be tied respectively to the existence of generalizable non-robust features \citep{jo2017measuring,ilyas2019adversarial,yin2019fourier} and to \emph{robust features} that are predictive for standard datasets, but not for the general vision tasks that these datasets aim to capture. Our experiments on MNIST show that such \emph{overly-robust} features exist. We further argue both formally and empirically that data augmentation may offer a solution to both excessive sensitivity and invariance.

\section{Norm-bounded Sensitivity and Invariance Adversarial examples}
\label{sec:definitions}

We begin by defining a framework to formally describe two complementary failure modes of machine learning models, namely (norm-bounded) adversarial examples that arise from excessive \emph{sensitivity} or \emph{invariance} of a classifier.

We consider a classification task with data $(x, y) \in \mathbb{R}^d \times \{1, \dots, C\}$ from a distribution $\mathcal{D}$. We assume the existence of a \emph{labeling oracle} $\mathcal{O} : \mathbb{R}^d \to \{1,\dots, C\} \cup \{\bot\}$ that maps any input in $\mathbb{R}^d$ to its true label, or to the ``garbage class'' $\bot$ for inputs $x$ considered ``un-labelable'' (e.g., for a digit classification task, the oracle $\mathcal{O}$ corresponds to human-labeling of any image as a digit or as the garbage class). Note that for $(x, y) \sim \mathcal{D}$, we always have $y= \mathcal{O}(x)$.\footnote{We view the support of $\mathcal{D}$ as a strict subset of all inputs in $\mathbb{R}^d$ to which the oracle assigns a label. That is, there are inputs for which humans agree on a label, yet that have measure zero in the data distribution from which the classifier's train and test inputs are chosen. For example, the train-test data is often a sanitized and normalized subset of natural inputs. Moreover, ``unnatural'' inputs such as adversarial examples might never arise in natural data.}

The goal of robust classification is to learn a classifier $f: \mathbb{R}^d \to \{1,\dots, C\}$ that agrees with the oracle's labels not only in expectation over the distribution $\mathcal{D}$, but also on any rare or out-of-distribution inputs to which the oracle assigns a class label---including adversarial examples obtained by imperceptibly perturbing inputs sampled from $\mathcal{D}$.

At its broadest, the definition of an adversarial example encompasses any adversarially induced failure in a classifier~\cite{goodfellow2017attacking}. That is, an adversarial example is any input $x^*$ created such that $f(x^*) \neq \mathcal{O}(x^*)$. This definition has proven difficult to work with, due to its inherent reliance on the oracle $\mathcal{O}$. As a result, it has become customary to study a relaxation of this definition, which restricts the adversary to applying a ``small'' perturbation to an input $x$ sampled from the distribution $\mathcal{D}$. A common choice is to restrict the adversary to small perturbations under some $\ell_p$-norm. We call these ``sensitivity adversarial examples'':
\begin{definition}[Sensitivity Adversarial Examples]
\label{def:advExam}
Given a classifier $f$ and a correctly classified input $(x, y) \sim \mathcal{D}$ (i.e., $\mathcal{O}(x) = f(x) = y$), an $\varepsilon$-bounded sensitivity adversarial example is an input $x^* \in \mathbb{R}^d$ such that:
\begin{myenumerate}
	\item $f(x^*) \neq f(x)$.
	\item $\|x^* - x\| \leq \varepsilon$.
\end{myenumerate} 
\end{definition}
The assumption underlying this definition is that perturbations satisfying $\|x^* - x\| \leq \varepsilon$ preserve the oracle's labeling of the original input $x$, i.e., $\mathcal{O}(x^*) = \mathcal{O}(x)$. 

A long line of work studies techniques to make classifiers robust to norm-bounded sensitivity adversarial examples~\cite{goodfellow2014explaining, madry2017towards}. The main objective is to minimize a classifier's \emph{robust error} under $\varepsilon$-bounded perturbations, which is defined as:
\begin{equation}
\label{eq:robust_err}
\mathcal{L}_\varepsilon(f) = \Exp\limits_{(x, y)\sim \mathcal{D}} \Big[\max_{\|\Delta\| \leq \varepsilon} \{f(x+\Delta) \neq y\}\Big].
\end{equation}
%
%
%
%
We study a complementary failure mode to sensitivity adversarial examples, called invariance adversarial examples~\citep{jacobsen2018excessive}. These correspond to (bounded) perturbations that 
\emph{do not} preserve an input's oracle-assigned label, yet preserve the model's classification:
\begin{definition}[Invariance Adversarial Examples]
	\label{def:advExamInv}
	Given a classifier $f$ and a correctly classified input $(x, y) \sim \mathcal{D}$, an $\varepsilon$-bounded invariance adversarial example is an input $x^* \in \mathbb{R}^d$ such that:
	\begin{myenumerate}
		\item $f(x^*) = f(x)$.
		\item $\mathcal{O}(x^*) \neq \mathcal{O}(x), \text{ and } \mathcal{O}(x^*) \neq \bot$.
		\item $\|x^* - x\| \leq \varepsilon$.
	\end{myenumerate} 
\end{definition}
If the assumption on sensitivity adversarial examples in Definition~\ref{def:advExam} is met---i.e., all $\varepsilon$-bounded perturbations preserve the label---then Definition~\ref{def:advExam} and Definition~\ref{def:advExamInv} correspond to well-separated failure modes of a classifier (i.e., 
$\varepsilon'$-bounded invariance adversarial examples only exist for $\varepsilon' > \varepsilon$).

Our main contribution is to reveal fundamental trade-offs between these two types of adversarial examples, that arise from this assumption being violated. 
We demonstrate that state-of-the-art robust classifiers do violate this assumption, and (sometimes certifiably) have low robust error $\mathcal{L}_\varepsilon(f)$ for a norm-bound $\varepsilon$ that does not guarantee that the oracle's label is preserved. We show that these classifiers actually have high ``true'' robust error as measured by human labelers. 

\vspace{-0.5em}
\paragraph{Remarks.}
Definition~\ref{def:advExamInv} is a conscious restriction on a definition of~\citet{jacobsen2018excessive}, who define an invariance adversarial example as an \emph{unbounded} perturbation that changes the oracle's label while preserving a classifier's output at \emph{an intermediate feature layer}. As we solely consider the model's final classification, considering unbounded perturbations would allow for a ``trivial'' attack: given an input $x$ of class $y$, find any input of a different class that the model misclassifies as $y$. (e.g., given an image of a digit $8$, an unbounded invariance example could be any unperturbed digit that the classifier happens to misclassify as an $8$).

Definition~\ref{def:advExamInv} presents the same difficulty as the original broad definition of adversarial examples: a dependence on the oracle $\mathcal{O}$.
Automating the process of finding invariance adversarial examples is thus challenging. In Section~\ref{sec:eval}, we present some successful automated attacks, but show that a human-in-the-loop process is more effective.

\begin{figure}[t]
    \centering
    \includegraphics[width=\columnwidth]{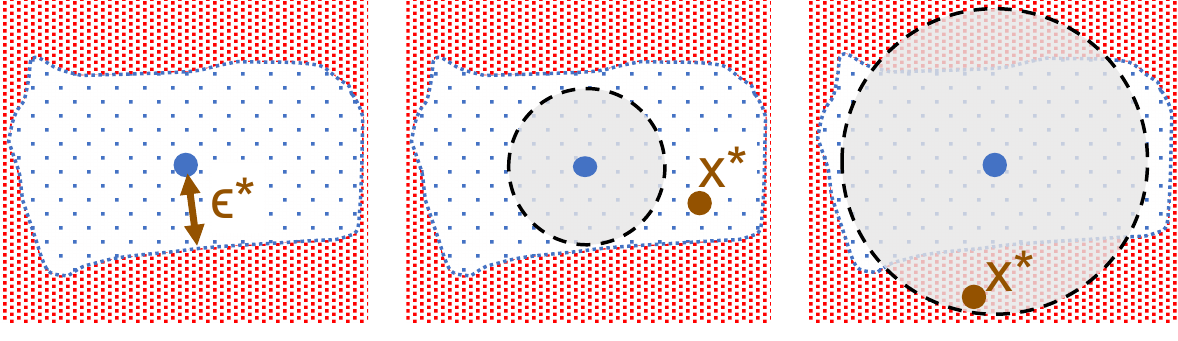}
    \vspace{-0.25em}
    (a) \hspace{0.27\columnwidth} (b) \hspace{0.27\columnwidth} (c) 
    \vspace{-0.25em}
    \caption{Illustration of distance-oracle misalignment. The input space is (ground-truth) classified into the red solid region, and the white dotted region. (a) A point at distance $\varepsilon^*$ (under a chosen norm) of the oracle decision boundary. (b) A model robust to perturbations of norm $\varepsilon \leq \varepsilon^*$ (gray circle) is still overly sensitive and can have adversarial examples $x^*$. (c) A model robust to perturbations of norm $\varepsilon > \varepsilon^*$ (gray circle) has invariance adversarial examples $x^*$.}
    \label{fig:alignment}
    \vspace{-0.75em}
\end{figure}

\section{The Sensitivity and Invariance Tradeoff}
\label{sec:tradeoff}

In this section, we show that if the norm that is used to define ``small'' adversarial perturbations is misaligned with the labeling oracle $\mathcal{O}$, then the robust classification objective in~\eqref{eq:robust_err} is insufficient for preventing both sensitivity-based and invariance-based adversarial examples under that norm. That is, we show 
that optimizing a model to attain low robust error on perturbations of norm $\varepsilon$ cannot prevent both sensitivity and invariance adversarial examples.

We begin by formalizing our notion of norm-oracle misalignment. The definition applies to any similarity metric over inputs, of which $\ell_p$-norms are a special case.

\begin{definition}[Distance-Oracle Misalignment]
    \label{def:align}
	Let $\texttt{dist}: \mathbb{R}^d \times \mathbb{R}^d \to \mathbb{R}$ be a distance measure (e.g., $\|x_1 - x_2\|$). We say that $\texttt{dist}$ is aligned with the oracle $\mathcal{O}$ if for any input $x$ with $\mathcal{O}(x) = y$, and any inputs $x_1, x_2$ such that $\mathcal{O}(x_1) = y, \; \mathcal{O}(x_2) \neq y $, we have $\texttt{dist}(x, x_1) < \texttt{dist}(x, x_2)$.  $\texttt{dist}$ and $\mathcal{O}$ are misaligned if they are not aligned.
\end{definition}

For natural images, $\ell_p$-norms (or other simple metrics) are clearly misaligned with our own perceptual metric. A concrete example is in Figure~\ref{fig:imageNetperturbations}.
This simple fact has deep implications for the suitability of the robust classification objective in~\eqref{eq:robust_err}. For an input $(x, y) \sim \mathcal{D}$, we define the size of the smallest class-changing perturbation as:
\begin{equation}
\label{eq:class-change-eps}
\varepsilon^*(x)  \coloneqq \min \left\{\|\Delta\|: \mathcal{O}(x+\Delta) \notin \{y, \bot\} \right\} \;.
\end{equation}
Let $x$ be an input where the considered distance function is not aligned with the oracle.
Let $x_2 = x + \Delta$ be the closest input to $x$ with a different class label, i.e., $\mathcal{O}(x_2) = y' \neq y$ and $\|\Delta\| = \varepsilon^*(x)$. As the distance and oracle are misaligned, there exists an input $x_1 = x + \Delta'$ such that $\|\Delta'\| > \varepsilon^*(x)$ and $\mathcal{O}(x_1) = y$.
So now, if we train a model to be robust (in the sense of ~\eqref{eq:robust_err}) to perturbations of norm bounded by $\varepsilon \leq \varepsilon^*(x)$, the model might misclassify $x_1$, i.e., it is sensitive to non-semantic changes. Instead, if we make the classifier robust to perturbations bounded by $\varepsilon > \varepsilon^*(x)$, then $x_2$ becomes an invariance adversarial examples as the model will classify it the same way as $x$.
The two types of failure modes are visualized in Figure~\ref{fig:alignment}.

\begin{lemma}
    \label{lemma:nn}
    Constructing an oracle-aligned distance function that satisfies Definition~\ref{def:align} is as hard
    as constructing a function $f$ so that $f(x) = \mathcal{O}(x)$, i.e., $f$ perfectly solves the oracle's classification task.
\end{lemma}
The proof of this lemma is in Appendix~\ref{apx:proof_nn}; at a high level, observe that
given a valid distance function that satisfies Definition~\ref{def:align} we can construct
a nearest neighbor classifier that perfectly matches the oracle.
Thus, in general we cannot hope to have such a distance function.

\begin{figure}[t]
    \centering
    \includegraphics[width=0.8\columnwidth]{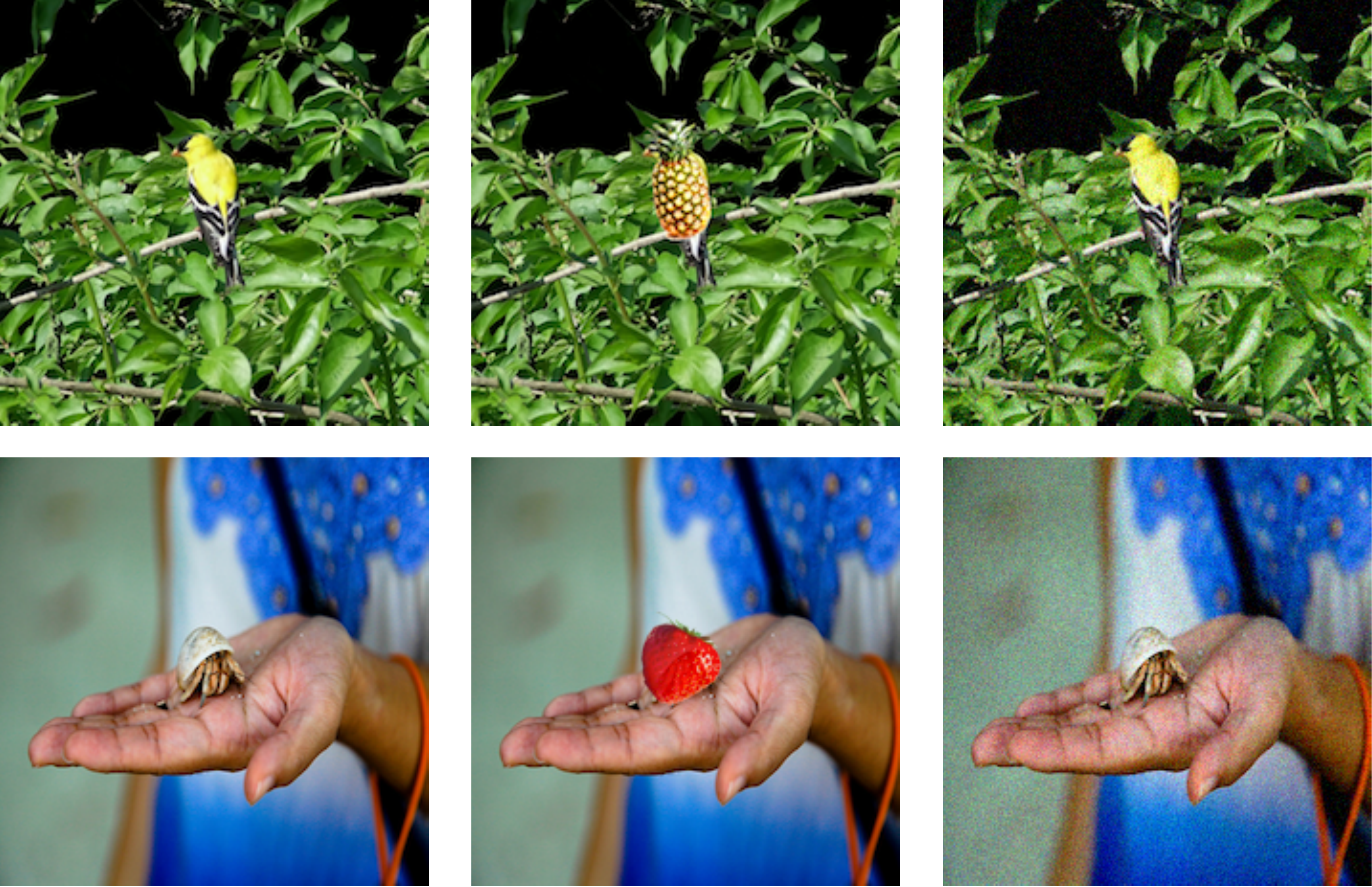} \\
    (a) \hspace{5em} (b) \hspace{5em} (c)
    \vspace{-0.5em}
    \caption{Visualization that $\ell_p$ norms fail to measure semantic similarity in images.\,\,\,\, (a) original image in the ImageNet validation set labeled as a \emph{goldfinch (top), hermit crab (bottom)}; (b) semantic perturbation with a $\ell_2$ perturbation of $19$ (respectively $22$) that replaces the object of interest with a pineapple (top), strawberry (bottom). (c) random perturbation of the same $\ell_2$-norm.}
    \label{fig:imageNetperturbations}
    \vspace{-0.75em}
\end{figure}

\section{Generating Invariance-based Adversarial Examples on MNIST} 
\label{sec:experiments}

We now empirically demonstrate and evaluate the trade-off between sensitivity-based and invariance-based adversarial examples. We propose an algorithm for generating invariance adversarial examples,
and show that robustified models are disparately more vulnerable to these attacks compared to standard models. In particular, we break both \emph{adversarially-trained} and \emph{certifiably-robust} models on MNIST by generating invariance adversarial examples---within the models' (possibly certified) norm bound---to which the models' 
assign different labels than an ensemble of humans.

\paragraph{Why MNIST?} 
We elect to study MNIST, the only dataset for which strong robustness to various $\ell_p$-bounded perturbations is attainable with current techniques~\citep{madry2017towards, schott2018towards}.
The dataset's simplicity is what initially prompted the study of simple $\ell_p$-bounded perturbations~\citep{harnessing_adversarial}. Increasing MNIST models' robustness to such perturbations has since become a standard benchmark~\citep{schott2018towards,madry2017towards,kolter2017provable,raghunathan2018certified}.
Due to the existence of models with high robustness to various $\ell_p$-bounded attacks, robust classification on MNIST is considered close to solved~\citep{schott2018towards}.

This paper argues that, contrary to popular belief, MNIST is far from being solved. 
We show why optimizing for robustness to $\ell_p$-bounded adversaries is not only insufficient, but actively harms the performance of the classifier against alternative invariance-based attacks.

In Section~\ref{ssec:natural}, we show that complex vision tasks (e.g., ImageNet) are also affected by the fundamental tradeoffs we describe. These tradeoffs are simply not apparent yet, because of our inability to train models with non-negligible robustness to any attacks on these tasks.

\subsection{Generating Model-agnostic Invariance-based Adversarial Examples}

We propose a model-agnostic algorithm for crafting invariance adversarial examples. Our attack generates minimally perturbed invariance adversarial examples that cause humans to change their classification.%
We then evaluate these examples against multiple models.
The rationale for this approach is mainly that obtaining human labels is expensive, which encourages the use of a single attack for all models. 

\newlength{\textfloatsepsave} \setlength{\textfloatsepsave}{\textfloatsep} \setlength{\textfloatsep}{1em}
\begin{algorithm}[t]
	\SetKwProg{GenInv}{GenInv}{}{}
    \SetKwInOut{Input}{Input}
    \SetKwInOut{Output}{Output}
    \GenInv{$(x, y, \mathcal{X},\mathcal{T})$}{
    	$\mathcal{S} = \{\hat{x}: (\hat{x}, \hat{y}) \in \mathcal{X}, \hat{y} \neq y\}$ \\
    	$\mathcal{X}^* = \{t(\hat{x}): t \in \mathcal{T}, \hat{x} \in \mathcal{S}\}$\\
    	\Return $x^* = \argmin_{\hat{x} \in \mathcal{X}^*} \|\hat{x} - x\|$
    }
    \caption{Meta-algorithm for finding invariance-based adversarial examples. For an input $x$, we find an input $x^*$ of a different class in the dataset $\mathcal{X}$, that is closest to $x$ under some set of semantics-preserving transformations. $\mathcal{T}$. 
    }
    \label{alg:geninv}
\end{algorithm}
\setlength{\textfloatsep}{\textfloatsepsave}

The high-level algorithm we use is in Algorithm~\ref{alg:geninv} and described below. It is simple, albeit tailored to datasets where comparing images in pixel space is meaningful, like MNIST.%
\footnote{\citet{kaushik2019learning} consider a similar problem for NLP tasks. They ask human labelers to produce ``counterfactually-augmented data'' by introducing a minimal number of changes to a text document so as to change the document's semantics.}

Given an input $x$, the attack's goal is to find the smallest class-changing perturbation $x^* = x+\Delta$ (c.f.~\eqref{eq:class-change-eps}) such that $\mathcal{O}(x^*) \neq \mathcal{O}(x)$. Typically, $x^*$ is not a part of the dataset. We thus approximate $x^*$ via  \emph{semantics-preserving} transformations of other inputs. That is, for the set $\mathcal{S}$ of inputs of a different class than $x$, we apply transformations $\mathcal{T}$ (e.g., small image rotations, translations) that are known a-priori to preserve input labels. We then pick the transformed input that is closest to our target point under the considered $\ell_p$ metric. In Appendix~\ref{app:attacksMNIST},
we describe instantiations of this algorithm for the $\ell_0$ and $\ell_\infty$ norms. Figure \ref{fig:l0attack} visualizes the sub-steps for the $\ell_0$ attack, including an extra post-processing that further reduces the perturbation size.

\paragraph{Measuring Attack Success.}
We refer to an invariance adversarial example as \emph{successful} if it causes a change in the oracle's label, i.e., $\mathcal{O}(x^*) \neq \mathcal{O}(x)$. This is a model-agnostic version of Definition~\ref{def:advExamInv}.
In practice, we simulate the oracle by asking an ensemble
of humans to label the point $x^*$; if more than some fraction
of them agree on the label (throughout this section, $70\%$) and that
label is different from the original, the attack is successful.
Note that success or failure is independent of any machine learning
model. 

\begin{figure}
    \centering
    \includegraphics[width=\columnwidth]{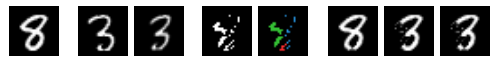} \\[-0.5em]
     (a) \hspace{1.5em} (b) \hspace{1em} (c) \hspace{1.5em} (d) \hspace{1em} (e) \hspace{4em} (f-h)  \hspace{2em}
    \vspace{-0.5em}
    \caption{Process for generating $\ell_0$ invariant adversarial examples. From left to right:
    (a) the original image of an 8;
    (b) the nearest training image (labeled as 3), before alignment;
    (c) the nearest training image (still labeled as 3), after alignment;
    (d) the $\Delta$ perturbation between the original and aligned training example;
    (e) spectral clustering of the perturbation $\Delta$; and
    (f-h) candidate invariance adversarial examples, selected by applying subsets of clusters of $\Delta$ to the original image. (f) is a failed attempt at an invariance adversarial example. (g) is successful, but introduces a larger perturbation than necessary (adding pixels to the bottom of the 3). (h) is successful and minimally perturbed.}
    \label{fig:l0attack}
    \vspace{-0.5em}
\end{figure}

\subsection{Evaluation}
\label{sec:eval}

\paragraph{Attack analysis.} We generate 100 invariance adversarial examples on inputs randomly drawn from the MNIST test set, for both the $\ell_0$ and $\ell_\infty$ norms. Our attack is slow, with the alignment process taking (amortized) minutes per example. We performed no optimizations of this process and expect it could be improved.
The mean $\ell_0$ distortion of successful examples is 25.9 (with a median of 25). 
The $\ell_\infty$ attack always uses the full budget of either $\varepsilon=0.3$ or $\varepsilon=0.4$ and runs in a similar amount of time.

\paragraph{Human Study.} We conducted a human study to evaluate whether our invariance adversarial examples are indeed successful, i.e., whether humans agree that the label has been changed.
We also hand-crafted $50$ invariance adversarial examples for the $\ell_0$ and $\ell_\infty$ norm. The process was quite simple: we built an image editor that lets us change images at a pixel level under an $\ell_p$ constraint. One author then modified 50 random test examples in the way that they perceived as changing the underlying class. 
We presented all these invariance examples to 40 human evaluators. Each evaluator classified 100 digits, half of which were unmodified MNIST digits, and the other half were sampled randomly from our $\ell_0$ and $\ell_\infty$ invariance adversarial examples. 

\begin{table}
\centering
    \begin{tabular}{@{}lr@{}}
    \toprule
         Attack Type & Success Rate \\
         \midrule
         Clean Images & 0\% \\
         \noalign{\smallskip}
         $\ell_0$ Attack & 55\%  \\
         \noalign{\smallskip}
         $\ell_\infty$, $\varepsilon=0.3$ Attack & 21\%  \\
         $\ell_\infty$, $\varepsilon=0.3$ Attack (\textbf{manual}) & 26\%  \\
         \noalign{\smallskip}
         $\ell_\infty$, $\varepsilon=0.4$ Attack & 37\%  \\
         $\ell_\infty$, $\varepsilon=0.4$ Attack (\textbf{manual}) & 88\%   \\
         \bottomrule
    \end{tabular}
    \vspace{-0.5em}
    \caption{Success rate of our invariance adversarial examples in causing humans to switch their classification.}
    \label{tab:humanstudy:eg}
    \vspace{-1em}
\end{table}%

\begin{figure}[t]
\centering
    \includegraphics[width=0.45\columnwidth]{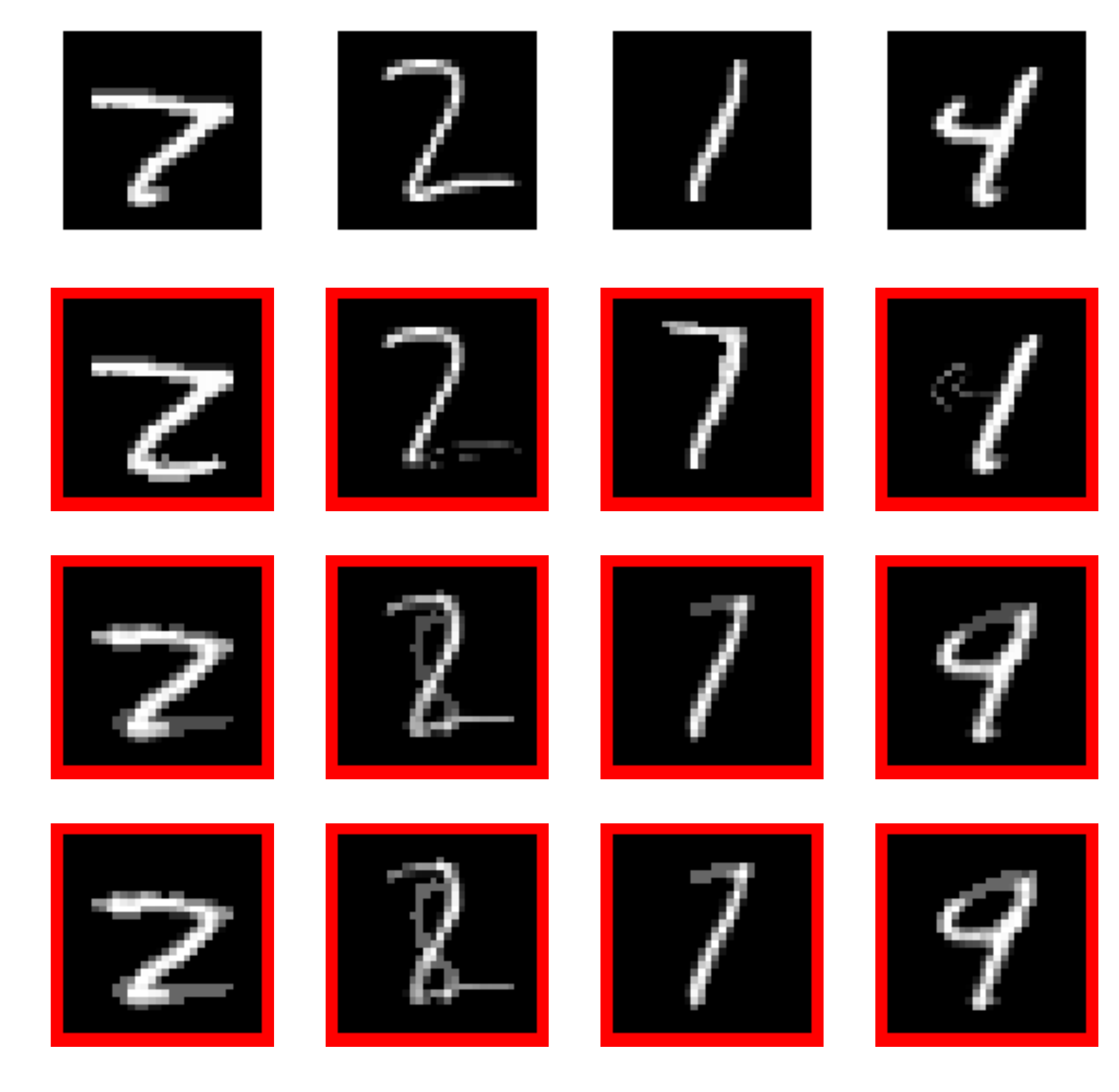}
    \hfill
    \includegraphics[width=0.45\columnwidth]{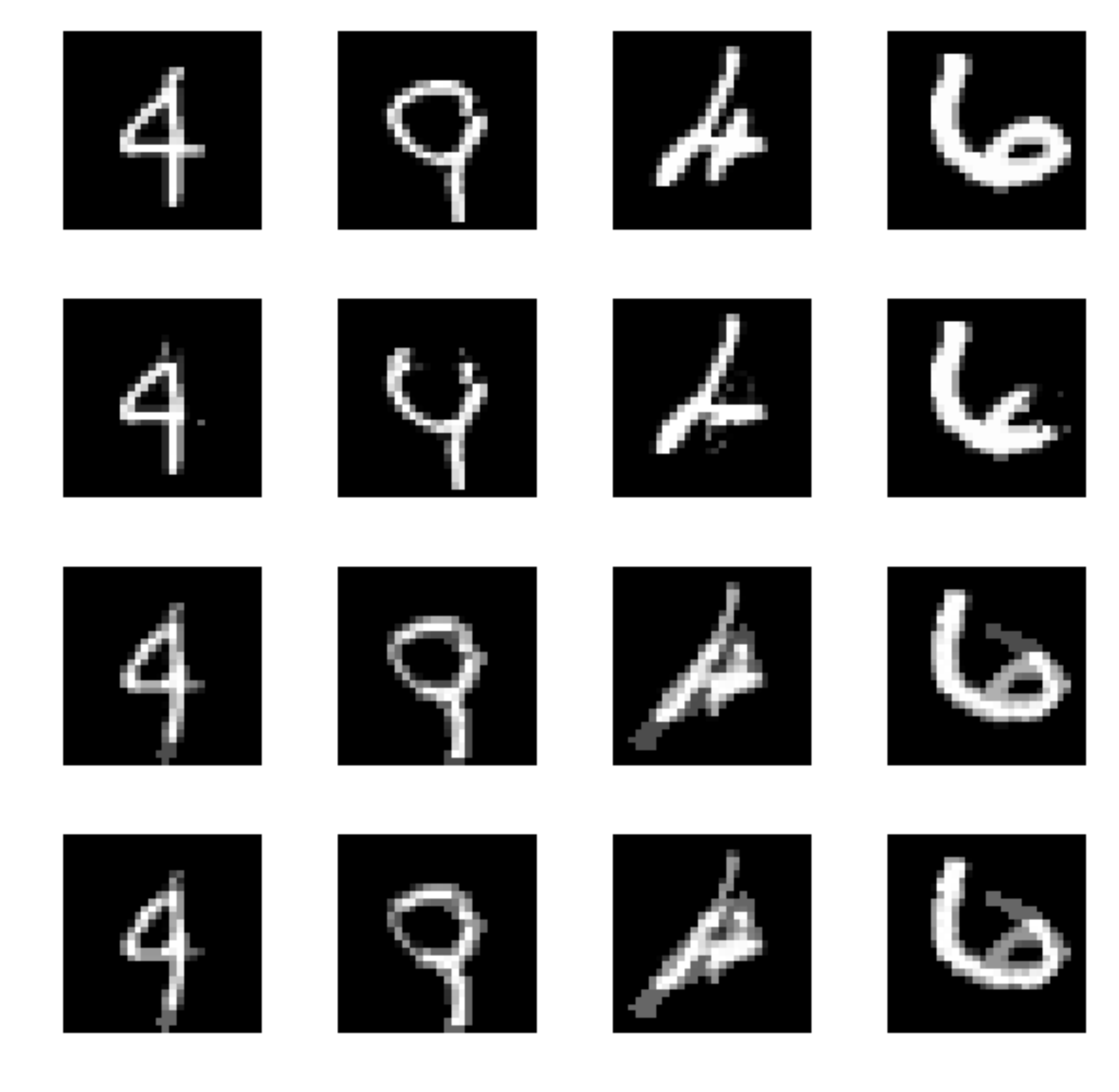}
    \vspace{-0.75em}
     \caption{Our invariance-based adversarial examples. Top to bottom: original images and our $\ell_0$, $\ell_\infty$ at $\varepsilon=0.3$ and $\ell_\infty$ at $\varepsilon=0.4$ invariance adversarial examples. (left) successful attacks; (right) failed attack attempts.}
    \label{fig:examples}
    \vspace{-0.75em}
\end{figure}

\begin{figure*}[t]
\centering
\begin{threeparttable}
    \centering
    \begin{tabularx}{\textwidth}{X c X}
    &\begin{tabular}{@{}l r r r r r r@{}}
    \toprule
          \multicolumn{7}{c}{Agreement between model and humans, for \emph{successful} invariance adversarial examples} \\
          \midrule
          \textbf{Model:}\tnote{1} & \textbf{Undefended} &   \textbf{$\ell_0$ Sparse}  & \textbf{Binary-ABS} & \textbf{ABS}  & \textbf{$\ell_\infty$ PGD} ($\epsilon=0.3$) & \textbf{$\ell_2$ PGD} ($\epsilon=2$)   \\
         \midrule
         Clean        &  99\%  & 99\% & 99\% & 99\% & 99\% & 99\% \\
         $\ell_0$     & 80\%  & 38\% & 47\% & 58\% & 56\%$^*$ & 27\%$^*$ \\
         $\ell_\infty$, $\varepsilon=0.3$ & 33\%  &  19\%$^*$ & 0\% & 14\% & 0\% & 5\%$^*$  \\
         $\ell_\infty$, $\varepsilon=0.4$ & 51\%&  27\%$^*$ & 8\% & 18\% & 16\%$^*$ & 19\%$^*$\\
         \bottomrule
    \end{tabular}
    &
    \end{tabularx}
    \begin{tablenotes}
    \item[1]\footnotesize{$\ell_0$ Sparse: \cite{bafna2018thwarting}; ABS and Binary-ABS: \cite{schott2018towards}; $\ell_\infty$ PGD and $\ell_2$ PGD: \cite{madry2017towards}}
    \end{tablenotes}
\end{threeparttable}
\vspace{-0.25em}
\captionof{table}{Model accuracy with respect to the oracle human labelers on
    the subset of examples where the human-obtained oracle label is
    different from the test label.
    Models which are \emph{more} robust to \emph{perturbation} adversarial examples (such as those trained with adversarial training) tend to agree with humans \textbf{less often} on \emph{invariance-based} adversarial examples. Values denoted with an asterisks $^*$ violate the perturbation threat model of the defense and should not be taken to be attacks. When the model is \emph{wrong}, it failed to classify the input as the new oracle label.}
    \label{tab:modelaccuracy}
\end{figure*}

\paragraph{Results.} Of 100 clean (unmodified) test images, 98 are labeled identically by \emph{all} human evaluators. The other 2 images were labeled identically by over $90\%$ of evaluators. 

Our $\ell_0$ attack is highly effective: For 55 of the 100 examples at least $70\%$ of human evaluators labeled it the same way, with a different label than the original test label. Humans only agreed with the original test label (with the same $70\%$ threshold) on 34 of the images, while they did not form a consensus on 18 examples.
The simpler $\ell_\infty$ attack is less effective: with a distortion of $0.3$ the oracle label changed $21\%$ of the time and with $0.4$ the oracle label changed $37\%$ of the time. The manually created $\ell_\infty$ examples with distortion of $0.4$ were highly effective however: for $88\%$ of the examples, at least $70\%$ assigned the same label (different than the test set label). We summarize results in Table~\ref{tab:humanstudy:eg}.
In Figure~\ref{fig:examples} we show sample invariance adversarial examples.

Our attack code, as well as our invariance examples and their human-assigned labels are available at {\footnotesize \url{https://github.com/ftramer/Excessive-Invariance}}.

To simplify the analysis below, we split our generated invariance adversarial examples into two sets: the successes and the failures, as determined by whether the plurality decision by humans was different than or equal to the original label. We only evaluate models on those invariance adversarial examples that caused the humans to switch their classification.

\paragraph{Model Evaluation.} Given oracle ground-truth labels for each of the images (as decided by humans), we report how often models agree with the human-assigned label.
Table~\ref{tab:modelaccuracy} summarizes this analysis. For the invariance adversarial examples, we report model accuracy only on \emph{successful} attacks (i.e., those where the human oracle label changed between the original image and the modified image).\footnote{It may seem counter-intuitive that our $\ell_\infty$ attack with $\varepsilon=0.3$ appears stronger than the one with $\varepsilon=0.4$. Yet, given two successful invariance examples (i.e., that both change the human-assigned label), the one with lower distortion is expected to change a model's output less often, and is thus a stronger invariance attack.}
For these same models, Table~\ref{tab:modelrobustness} in Appendix~\ref{app:hyperparams} reports the ``standard'' robust accuracy for sensitivity-based adversarial examples, i.e., in the sense of~\eqref{eq:robust_err}.

The models which empirically achieve the highest robustness against $\ell_0$ perturbations (in the sense of~\eqref{eq:robust_err}) are the $\ell_0$-Sparse classifier of \citet{bafna2018thwarting}, the Binary-ABS model of~\citet{schott2018towards}, and the $\ell_2$-PGD adversarially trained model (see Table~\ref{tab:modelrobustness} in the Appendix for a comparison of the robustness of these models). Thus, these are the models that are most \emph{invariant} to perturbations of large $\ell_0$-norm. We find that these are the models that achieve the lowest accuracy---as measured by the human labelers---on our invariance examples. Moreover, all robust models perform much worse than an undefended ResNet-18 model on our invariance attacks. This includes models such as the $\ell_\infty$-PGD adversarially trained model, which do not explicitly aim at worst-case robustness against $\ell_0$ noise. 
Thus, we find that models that were designed to reduce excessive sensitivity to certain non-semantic features, become excessively invariant to other features that are semantically meaningful.

Similarly, we find that models designed for $\ell_\infty$ robustness (Binary-ABS and $\ell_\infty$-PGD) also fare the worst on our $\ell_\infty$ invariance adversarial examples. Overall, all robust models do worse than the undefended baseline. The results are consistent for attacks with $\eps=0.3$ and with $\eps=0.4$, the latter being more successful in changing human labels.

Note that the Binary-ABS defense of~\cite{schott2018towards} boasts $60\%$ (empirical) robust accuracy on $\ell_\infty$-attacks with $\varepsilon=0.4$ (see~\cite{schott2018towards}). Yet, on our our invariance examples that satisfy this perturbation bound, the model actually disagrees with the human labelers $92\%$ of the time, and thus achieves only $8\%$ true accuracy on these examples. Below, we make a similar observation for a \emph{certified} defense.

\paragraph{Trading Perturbation-Robustness for Invariance-Robustness.}

To better understand how robustness to sensitivity-based adversarial examples influences robustness to invariance attacks, we evaluate a range of adversarially-trained models on our invariance examples.

Specifically, we trained $\ell_\infty$-PGD models with $\varepsilon \in [0, 0.4]$ and $\ell_1$-PGD models (as a proxy for $\ell_0$-robustness) with $\varepsilon \in [0, 15]$. 
We verified that training against larger perturbations resulted in a monotonic increase in adversarial robustness, in the sense of~\eqref{eq:robust_err} (more details are in Appendix~\ref{app:hyperparams}).
We then evaluated these models against respectively the $\ell_\infty$ and $\ell_0$ invariance examples.
Figure~\ref{fig:pgd_invariance} shows that robustness to larger perturbations leads to higher vulnerability to invariance-based examples.

Interestingly, while sensitivity-based robustness does not generalize beyond the norm-bound on which a model is trained (e.g., a model trained on PGD with $\varepsilon=0.3$ achieves very little robustness to PGD with $\varepsilon=0.4$~\cite{madry2017towards}), excessive invariance does generalize (e.g., a model trained on PGD with $\varepsilon=0.2$ is more vulnerable to our invariance attacks with $\varepsilon \geq 0.3$ compared to an undefended model).

\begin{figure}[t]
    \centering
    \includegraphics[width=\columnwidth]{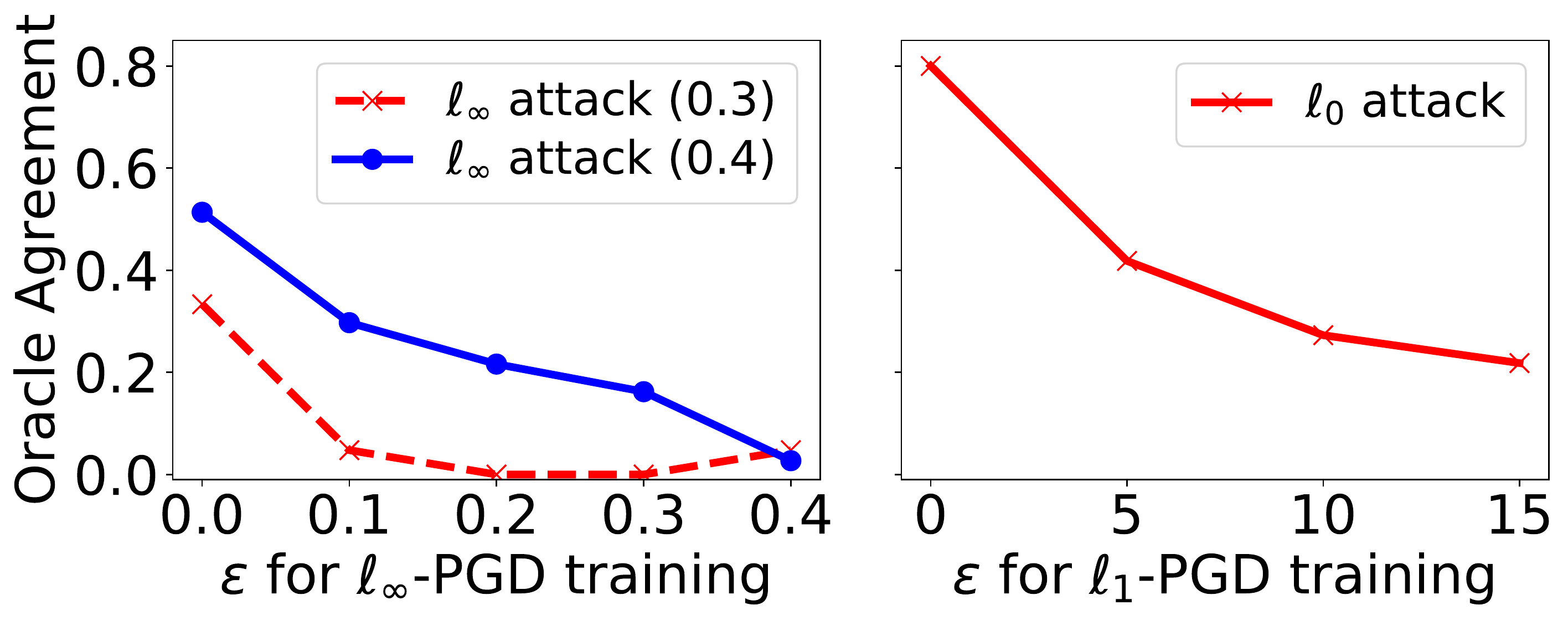}
    
    \vspace{-0.75em}
    \caption{Higher noise-robustness leads to higher vulnerability to invariance attacks. (left) For models trained with $\ell_\infty$-PGD, a higher bound $\varepsilon \in [0, 0.4]$ implies lower accuracy on $\ell_\infty$-bounded invariance examples. (right) Models trained with $\ell_1$-PGD evaluated on
    the $\ell_0$ invariance attack.}
    \label{fig:pgd_invariance}
    \vspace{-0.75em}
\end{figure}

\paragraph{Breaking Certified Defenses.}
Our invariance attacks even constitute
a \emph{break} of some certified defenses. 
For example, \citet{zhang2019towards}
develop a defense which \emph{proves} that the accuracy on the test set is at least $87\%$ under $\ell_\infty$ perturbations
of size $\varepsilon=0.4$.
When we run their pre-trained model on all 100 of our $\varepsilon=0.4$
invariance adversarial
examples (i.e., not just the successful ones) we find it has a $96\%$ ``accuracy'' (i.e., it matches
the original test label $96\%$ of the time).
However, when we look at the agreement between this model's predictions with the new labels assigned by the human evaluators, the model's accuracy is just $63\%$.

Thus, while the proof in the paper is \emph{mathematically correct}
it does not actually deliver $87\%$ robust accuracy under $\ell_\infty$-attacks with $\varepsilon=0.4$: humans change their classification for many of these perturbations.
Worse, for the 50 adversarial examples we crafted by hand, the model
\emph{disagrees} with the human ensemble $88\%$ of the time:
it has just $12\%$ accuracy.

\subsection{Natural Images}
\label{ssec:natural}

While our experiments are on MNIST, similar phenomena may arise in other vision tasks. Figure~\ref{fig:imageNetperturbations} shows two perturbations of ImageNet images: the rightmost perturbation is imperceptible and thus classifiers should be robust to it. Conversely, the middle image was semantically changed, and classifiers should be sensitive to such changes. Yet, the $\ell_2$ norm of both perturbations is the same. 
Hence, enforcing robustness to $\ell_2$-noise of some fixed size $\varepsilon$ will necessarily result in a classifier that is either sensitive to the changes on the right, or invariant to the changes in the middle image.
Such a phenomenon will necessarily arise for any image dataset that contains small objects, as perturbations of small $\ell_2$ magnitude will be sufficient to occlude the object, thereby changing the image semantics.

This distance-oracle misalignment extends beyond the $\ell_2$-norm. For instance, \citet{co2018procedural} show that a perturbation of size 16/255 in $\ell_\infty$ can suffice to give an image of a cat the appearance of a shower curtain print, which are both valid ImageNet classes. Yet, a \emph{random} perturbation of the same magnitude is semantically meaningless.

On CIFAR-10, some recent defenses are possibly already overly invariant. For example, \citet{shaeiri2020towards} and \citet{panda2019discretization} aim to train models that are robust to $\ell_\infty$ perturbations of size $\epsilon=32/255$. Yet, \citet{tsipras2018robustness} show that perturbations of that magnitude can be semantically meaningful and can be used to effectively \emph{interpolate} between CIFAR-10 classes. The approach taken by
~\citet{tsipras2018robustness} to create these perturbations, which is based on a model with robustness to very small $\ell_\infty$ noise, may point towards an efficient way of automating the generation of invariance attacks for tasks beyond MNIST.
The work of~\citet{sharif2018suitability} also shows that ``small'' $\ell_\infty$ noise (of magnitude $25/255$) can reliably fool human labelers on CIFAR-10.

\section{The Overly-Robust Features Model}
\label{sec:robust_features}

The experiments in Section~\ref{sec:experiments} show that models can be robust to perturbations large enough to change an input's semantics. Taking a step back, it is not obvious why training such classifiers is possible, i.e., why does excessive invariance not harm regular accuracy.
To understand the learning dynamics of these overly-robust models, we ask two questions:
\begin{myenumerate}
	\item Can an overly-robust model fit the training data?
	\item Can such a model generalize (robustly) to test data?
\end{myenumerate}
For simplicity, we assume that for every point $(x, y) \sim \mathcal{D}$, the closest point $x^*$ (under the chosen norm) for which $\mathcal{O}(x^*) \neq y$ is at a constant distance $\varepsilon^*$. We train a model $f$ to have low robust error (as in~\eqref{eq:robust_err}) for perturbations of size $\varepsilon > \varepsilon^*$. This model is thus overly-robust.

We first ask under what conditions $f$ may have low robust training error. 
A necessary condition is that there do not exist training points $(x_i, y_i), (x_j, y_j)$ such that $y_i \neq y_j$ and $\|x_i - x_j\| \leq \varepsilon$. As $\varepsilon$ is larger than the inter-class distance, the ability to fit an overly robust model thus relies on the training data not being fully representative of the space to which the oracle assigns labels.
This seems to be the case in MNIST: as the dataset consists of centered, straightened and binarized digits, even an imaginary infinite-sized dataset might not contain our invariance adversarial examples.

The fact that excessive robustness \emph{generalizes} (as provably evidenced by the model of~\citet{zhang2019towards})  points to a deeper issue: there must exist \emph{overly-robust} and \emph{predictive} features in the data---that are not aligned with human perception. 
This mirrors the observations of~\cite{ilyas2019adversarial}, who show that excessive sensitivity is caused by non-robust yet predictive features. On MNIST, our experiments confirm the existence of overly-robust generalizable features.

We formalize these observations using a simple classification task inspired by~\cite{tsipras2018robustness}. We consider a binary task where unlabeled inputs $x \in \mathbb{R}^{d+2}$ are sampled from a distribution $\mathcal{D}_k^*$ with parameter $k$:
\begin{gather*}
z \stackrel{\text{u.a.r}}{\sim} \{-1, 1\}, 
\quad 
x_1 = z/2
\\
x_2 = \begin{cases}
+z & \text{w.p. } \frac{1 + 1/k}{2} \\
-z & \text{w.p. } \frac{1 - 1/k}{2}
\end{cases}
,\ x_3,  \dots, x_{d+2} \stackrel{\text{i.i.d}}{\sim} \mathcal{N}(\frac{z}{\sqrt{d}}, k) \;.
\end{gather*}
Here $\mathcal{N}(\mu, \sigma^2)$ is a normal distribution and $k > 1$ is a constant chosen so that only feature $x_1$ is strongly predictive of the latent variable $z$ (e.g., $k=100$ so that $x_2, \dots, x_{d+2}$ are almost uncorrelated with $z$).
The oracle is defined as $\mathcal{O}(x) = \texttt{sign}(x_1)$, i.e., feature $x_1$ fully defines the oracle's class label, and other features are nearly uncorrelated with it. Note that the oracle's labels are robust under any $\ell_\infty$-noise with norm strictly below $\varepsilon=1/2$.

We model the collection of ``sanitized'' and labeled datasets from a data distribution as follows: the semantic features (i.e., $x_1$) are preserved, while ``noise'' features have their variance reduced (e.g., because non-standard inputs are removed). Sanitization thus enhances ``spurious correlations''~\cite{jo2017measuring, ilyas2019adversarial} between non-predictive features and class labels.\footnote{In digit classification for example, the number of pixels above $1/2$ is a feature that is presumably very weakly correlated with the class $8$. In the MNIST dataset however, this feature is fairly predictive of the class $8$ and robust to $\ell_\infty$-noise of size $\varepsilon=0.4$.}
We further assume that the data labeling process introduces some small label noise.\footnote{This technicality avoids that classifiers on $\mathcal{D}$ can trivially learn the oracle labeling function. Alternatively, we could define feature $x_1$ so that is is hard to learn for certain classes of classifiers.}
Specifically, the labeled data distribution $\mathcal{D}$ on which we train and evaluate classifiers is obtained by sampling $x$ from a sanitized distribution 
$\mathcal{D}^*_{1+\alpha}$ (for a small constant $\alpha>0$) where features $x_2, \dots, x_{d+2}$ are strongly correlated with the oracle label. The label $y$ is set to the correct oracle label with high probability $1-\delta$.
The consequences of this data sanitization are two-fold (see Appendix~\ref{app:proofs} for proofs):
\begin{myenumerate}
	\item A standard classifier (that maximizes accuracy on $\mathcal{D}$) agrees with the oracle with probability at least $1-\delta$, but is vulnerable to $\ell_\infty$-perturbations of size $\varepsilon=O(d^{-1/2})$.
	\item There is an overly-robust model that only uses feature $x_2$ and has robust accuracy $1-\alpha/2$ on $\mathcal{D}$ for $\ell_\infty$-noise of size $\varepsilon = 0.99$. This classifier is vulnerable to invariance attacks as the oracle is not robust to such perturbations.
\end{myenumerate}

\paragraph{The Role of Data Augmentation.}

This simple task suggests a natural way to prevent the training of overly robust models. If \emph{prior knowledge} about the task suggests that classification should be invariant to features $x_2 \dots, x_{d+2}$, then enforcing these invariances would prevent a model from being robust to excessively large perturbations.

A standard way to enforce invariances is via data augmentation.
In the above binary task, augmenting the training data by randomizing over features $x_2, \dots, x_{d+2}$ would force the model to rely on the only truly predictive feature, $x_1$.

\begin{figure}[t]
    \centering
    \includegraphics[width=0.7\columnwidth]{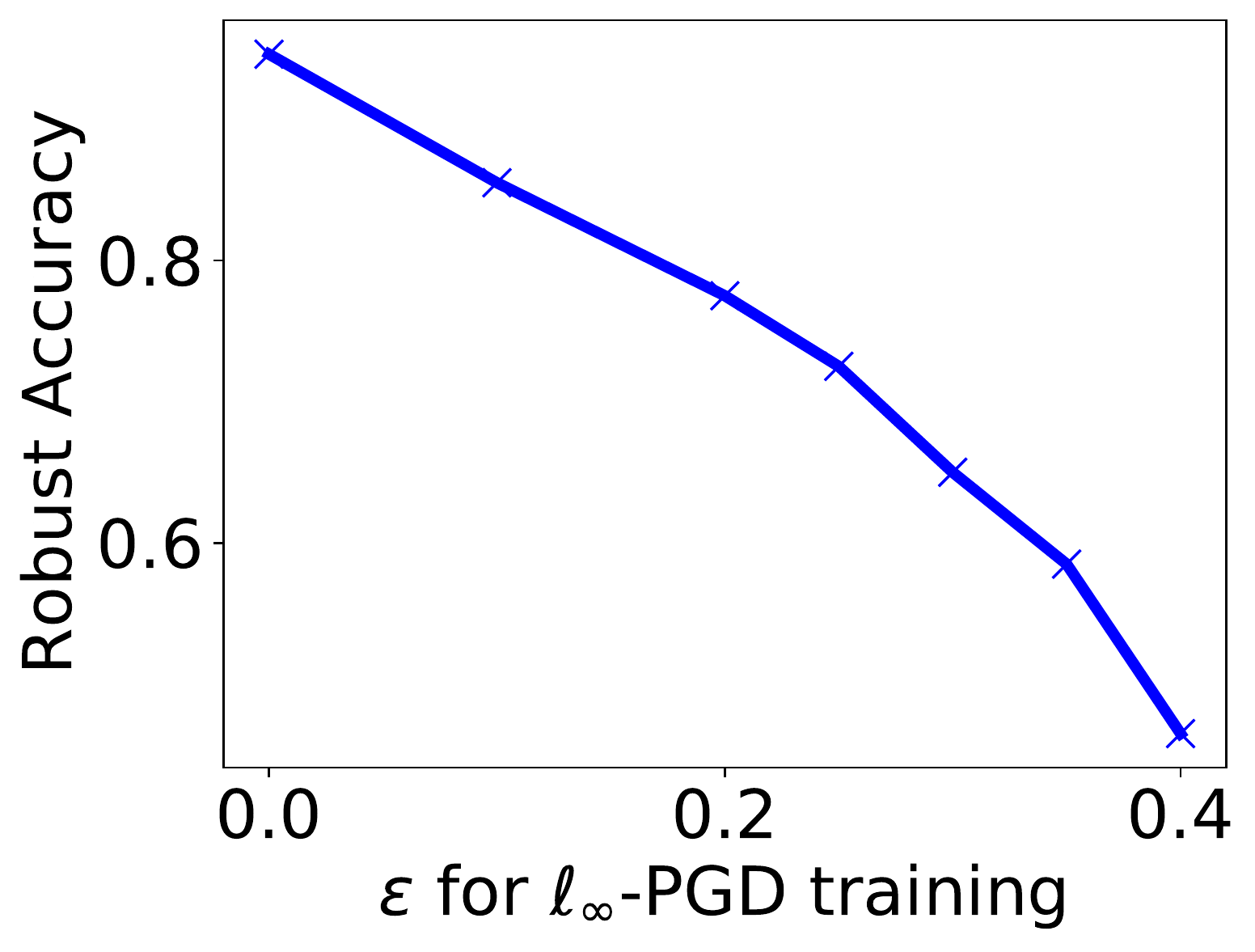}
    \vspace{-0.75em}
    \caption{Robust accuracy of models trained and evaluated on an adversary combining a small spatial data augmentation (rotation + translation) with an $\ell_\infty$ perturbation bounded by $\varepsilon$.}
    \vspace{-0.75em}
    \label{fig:RT_eps_models}
\end{figure}

\begin{figure}[t]
    \centering
    \includegraphics[width=0.8\columnwidth]{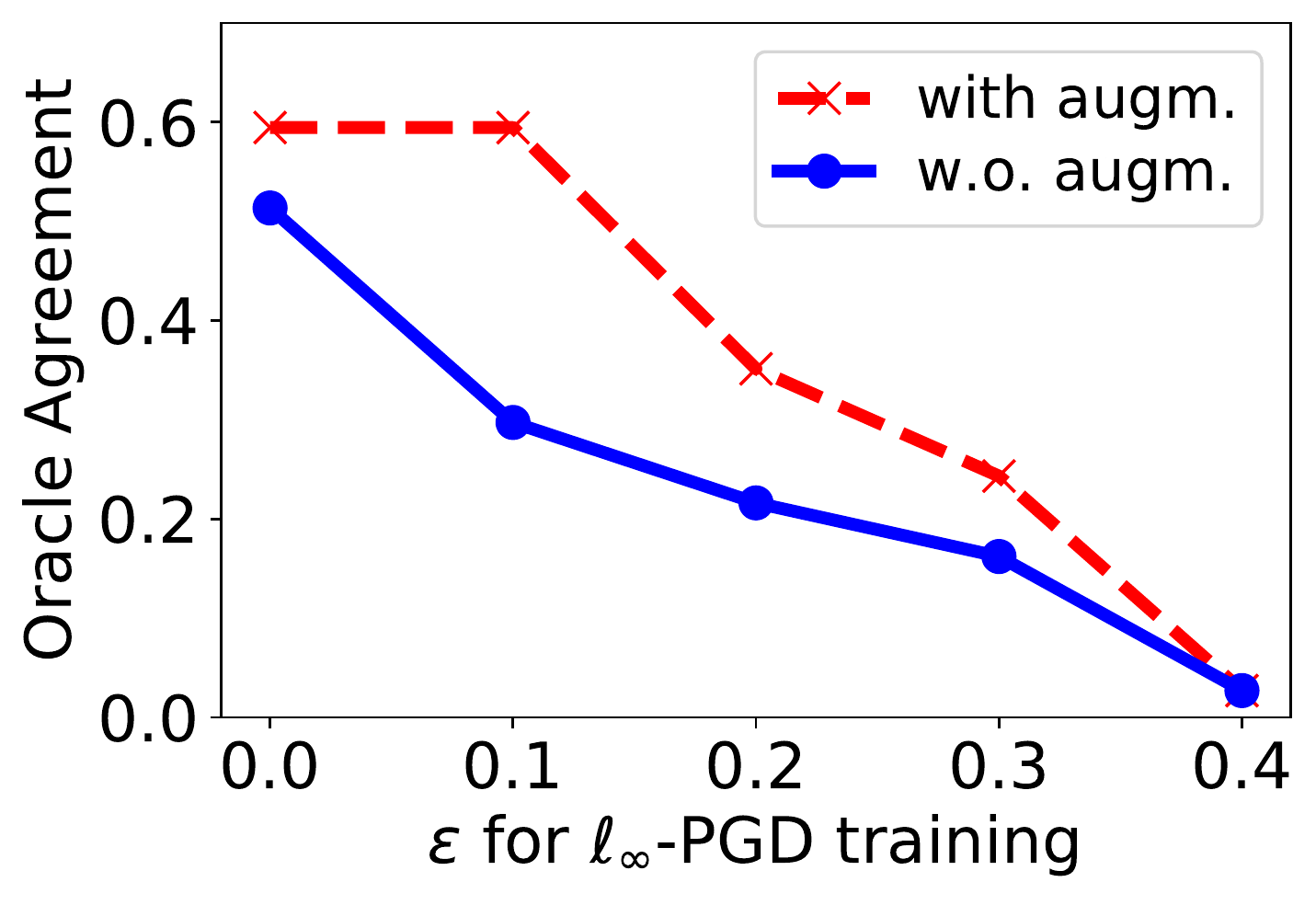}
    \vspace{-0.75em}
    \caption{Model accuracy with respect to human labelers on successful invariance adversarial examples of $\ell_\infty$-norm bounded by $0.4$. Models trained with data augmentation agree more often with humans, and thus are less invariant to semantically-meaningful changes.}
    \vspace{-0.75em}
    \label{fig:data_aug}
\end{figure}

We experimented with aggressive data-augmentation on MNIST. For values of $\varepsilon \in [0, 0.4]$, we train models with an adversary that rotates and translates inputs by a small amount and then adds $\varepsilon$-bounded $\ell_\infty$-noise. This attacker mirrors the process we use to generate invariance adversarial examples in Section~\ref{sec:experiments}.
Thus, we expect it to be hard to achieve robustness to attacks with large $\varepsilon$ on this dataset, as this requires the model to correctly classify inputs that humans consider mislabeled.
Figure~\ref{fig:RT_eps_models} confirms this intuition.
As $\varepsilon$ grows, it becomes harder to learn a model that is invariant to both spatial data augmentation and $\ell_\infty$ noise.
We further find that the models trained with data augmentation agree more often with human labelers on our invariance attacks (see Figure~\ref{fig:data_aug}). Yet, even with data augmentation, models trained against large $\ell_\infty$-perturbations still perform worse than an undefended model. 
This simple experiment thus demonstrates that while data-augmentation (over truly invariant features) can help in detecting or preventing excessive invariance to semantic features, even though it is not currently sufficient for training models that resist both sensitivity-based and invariance-based attacks. 

\section{Discussion}

Our results show that solely focusing on robustness to sensitivity-based attacks is insufficient, as mis-specified bounds can cause vulnerability to invariance-based attacks.

\paragraph{On $\ell_p$-norm evaluations.} Our invariance attacks are able to find points within the $\ell_p$-ball in which state-of-the-art classifiers 
are (provably) robust.
This highlights the need for a more careful selection of perturbation bounds when measuring robustness to adversarial examples. At the same time, Figure~\ref{fig:pgd_invariance} shows that even promoting robustness within conservative bounds causes excessive invariance.
The tradeoff explored in Section~\ref{sec:tradeoff} suggests that aiming for robustness against $\ell_p$-bounded attacks may be inherently futile for making models robust to arbitrary adversarial examples. 

\paragraph{Trading Sensitivity and Invariance.}

We show that models that are robust to small perturbations make excessively invariant decisions and are thus vulnerable to other attacks.

Interestingly, 
\citet{engstrom2019adversarial} 
show an opposite effect for models' internal representations.
Denoting the logit layer of a model as $z(x)$, they show that for robust models it is hard to find inputs $x, x^*$ such that $\mathcal{O}(x) \neq \mathcal{O}(x^*)$ and $z(x) \approx z(x^*)$. Conversely, \citet{sabour2015adversarial} and \citet{jacobsen2018excessive} show that excessive invariance of feature layers is common in non-robust models.
These observations are orthogonal to ours as we study invariances in a model's classification layer, and for bounded perturbations.
As we show in Section~\ref{sec:tradeoff}, robustness to large perturbations under a norm that is misaligned with human perception necessarily causes excessive invariance of the model's classifications (but implies nothing about the model's feature layers).

Increasing model robustness to $\ell_p$-noise also leads to other tradeoffs, such as reduced accuracy~\citep{tsipras2018robustness} or reduced robustness to other small perturbations~\citep{yin2019fourier, tramer2019adversarial, kang2019testing}.

\section{Conclusion}

We have introduced and studied a fundamental tradeoff between two types of adversarial examples, that stem either from excessive sensitivity or invariance of a classifier. This tradeoff is due to an inherent misalignment between simple robustness notions and a task's true perceptual metric.

We have demonstrated that 
defenses against $\ell_p$-bounded perturbations on MNIST promote invariance to semantic changes. Our attack exploits this excessive invariance by changing image semantics while preserving model decisions. For  adversarially-trained and certified defenses, our attack can reduce a model's true accuracy to random guessing.

Finally, we have studied the tradeoff between sensitivity and invariance in a theoretical setting where excessive invariance can be explained by the existence of overly-robust features.

Our results highlight the need for a more principled approach in selecting meaningful robustness bounds and in measuring progress towards more robust models.

\section*{Acknowledgments}

Florian Tramèr's research was supported in part by the Swiss National Science Foundation (SNSF project P1SKP2\_178149).
Resources used in preparing this research were provided, in part, by the Province of Ontario, the Government of Canada through CIFAR, and companies sponsoring the Vector Institute \url{www.vectorinstitute.ai/partners}.

\bibliography{biblio}
\bibliographystyle{icml2020}

\clearpage
\appendix
\appendix

\section{Details about Model-agnostic Invariance-based Attacks}
\label{app:attacksMNIST}
Here, we give details about our model-agnostic invariance-based adversarial attacks on MNIST. 

\paragraph{Generating $\ell_0$-invariant adversarial examples.}
Assume we are given a training set $\mathcal{X}$ consisting of labeled example pairs $(\hat{x},\hat{y})$. As input our algorithm accepts an example $x$ with oracle label $\mathcal{O}(x) = y$. Image $x$ with label $y=8$ is given in Figure~\ref{fig:l0attack} (a).

Define $\mathcal{S} = \{\hat{x} : (\hat{x},\hat{y}) \in \mathcal{X}, \hat{x} \ne y\}$, the set of training examples with a different label. Now we define $\mathcal{T}$ to be the set of transformations that we allow: rotations by up to $20$ degrees, horizontal or vertical shifts by up to $6$ pixels (out of 28), shears by up to $20\%$, and re-sizing by up to $50\%$.

We generate a new augmented training set $\mathcal{X}^* = \{t(\hat{x}) : t \in \mathcal{T}, \hat{x} \in \mathcal{S}\}$. By assumption, each of these examples is labeled correctly by the oracle. In our experiments, we verify the validity of this assumption through a  human study and omit any candidate adversarial example that violates this assumption. Finally, we search for
\[x^* = \mathop{\text{arg min}}\limits_{x^* \in \mathcal{X}^*} \lVert{}x^* - \hat{x}\rVert{}_0. \]
By construction, we know that $x$ and $x^*$ are similar in pixel space but have a different label. Figure~\ref{fig:l0attack} (b-c) show this step of the process.
Next, we introduce a number of refinements to make $x^*$ be ``more similar'' to $x$. This reduces the $\ell_0$ distortion introduced to create an invariance-based adversarial example---compared to directly returning $x^*$ as the adversarial example. 

First, we define $\Delta = |x-x^*|>1/2$ where the absolute value and comparison operator are taken element-wise.  Intuitively, $\Delta$ represents the pixels that substantially change between $x^*$ and $x$. We choose $1/2$ as an arbitrary threshold representing how much a pixel changes before we consider the change ``important''. This step is shown in Figure~\ref{fig:l0attack} (d).
Along with $\Delta$ containing the \emph{useful} changes that are responsible for changing the oracle class label of $x$, it also contains irrelevant changes that are superficial and do not contribute to changing the oracle class label. For example, in Figure~\ref{fig:l0attack} (d) notice that the green cluster is the only semantically important change; both the red and blue changes are not necessary.

To identify and remove the superficial changes, we perform spectral clustering on $\Delta$. We compute $\Delta_i$ by enumerating all possible subsets of clusters of pixel regions. This gives us many possible \textbf{potential} adversarial examples $x^*_i = x+\Delta_i$. Notice these are only potential because we may not actually have applied the necessary change that actually modifies the class label. 

We show three of the eight possible candidates in Figure~\ref{fig:l0attack}.
In order to alleviate the need for human inspection of each candidate $x^*_i$ to determine which of these potential adversarial examples is actually misclassified, we follow an approach from Defense-GAN \cite{samangouei2018defense} and the Robust Manifold Defense \cite{ilyas2017robust}: we take the generator from a GAN and use it to assign a likelihood score to the image. We make one small refinement, and use an AC-GAN \cite{mirza2014conditional} and compute the class-conditional likelihood of this image occurring. This process reduces $\ell_0$ distortion by $50\%$ on average.

As a small refinement, we find that initially filtering $\mathcal{X}$ by removing the $20\%$ least-canonical examples makes the attack succeed more often.

\paragraph{Generating $\ell_\infty$-invariant adversarial examples.}
Our approach for generating $\ell_\infty$-invariant examples follows similar ideas as for the $\ell_0$ case, but is conceptually simpler as the perturbation budget can be applied independently for each pixel (our $\ell_\infty$ attack is however less effective than the $\ell_0$ one, so further optimizations may prove useful).

We build an augmented training set $\mathcal{X}^*$ as in the $\ell_0$ case. Instead of looking for the closest nearest neighbor for some example $x$ with label $\mathcal{O}(x) = y$, we restrict our search to examples $x^* \in \mathcal{X}^*$ with specific target labels $y^*$, which we've empirically found to produce more convincing examples (e.g., we always match digits representing a $1$, with a target digit representing either a $7$ or a $4$). We then simply apply an $\ell_\infty$-bounded perturbation to $x$ by interpolating with $x^*$, so as to minimize the distance between $x$ and the chosen target example $x^*$.

\newpage
\section{Complete Set of 100 Invariance Adversarial Examples}
\label{app:inv_examples}

Below we give the $100$ randomly-selected test images along with the invariance adversarial examples that were shown during the human study.

\subsection{Original Images}
\includegraphics[width=\columnwidth]{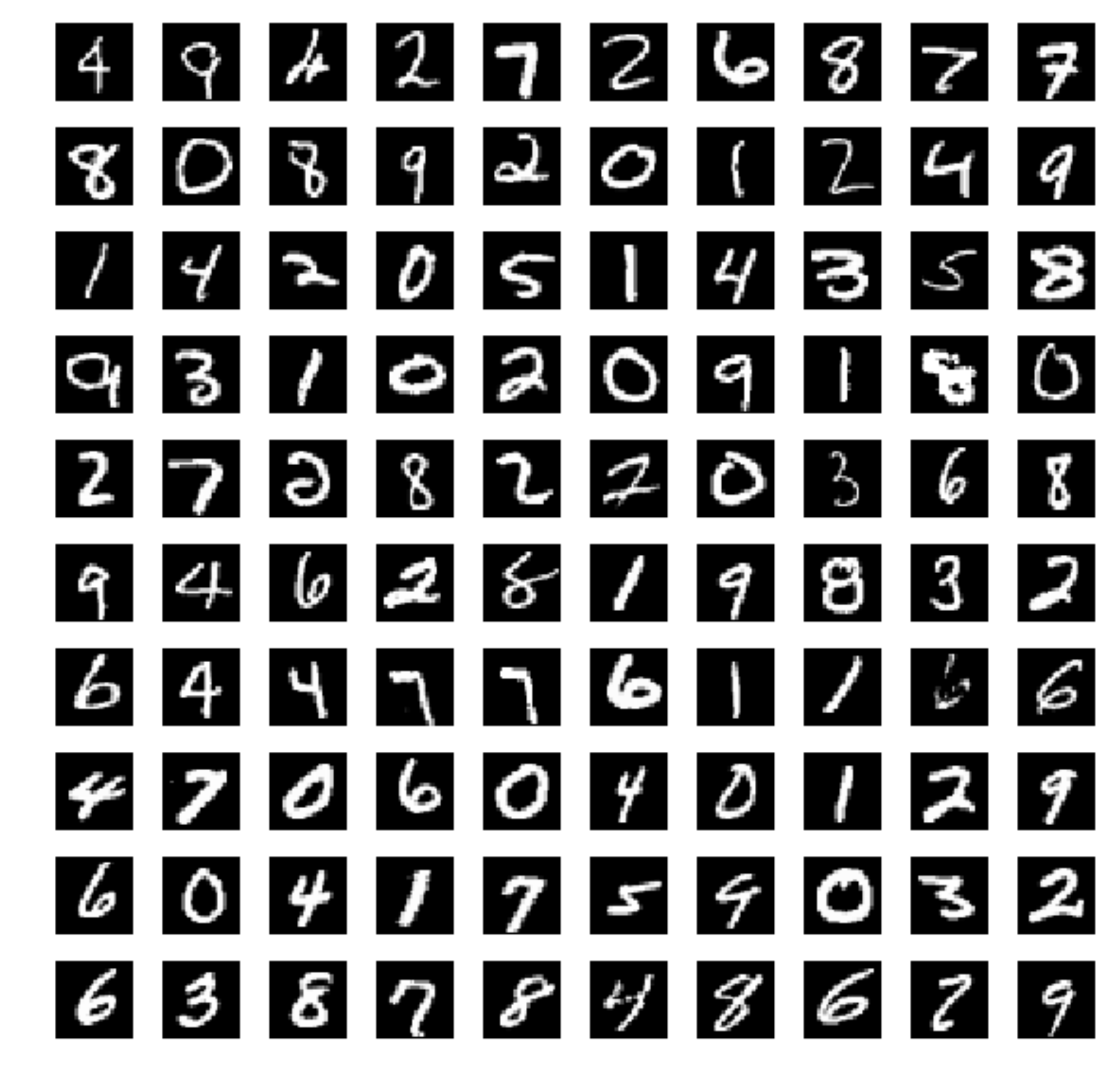}

\subsection{$\ell_0$ Invariance Adversarial Examples}
\includegraphics[width=\columnwidth]{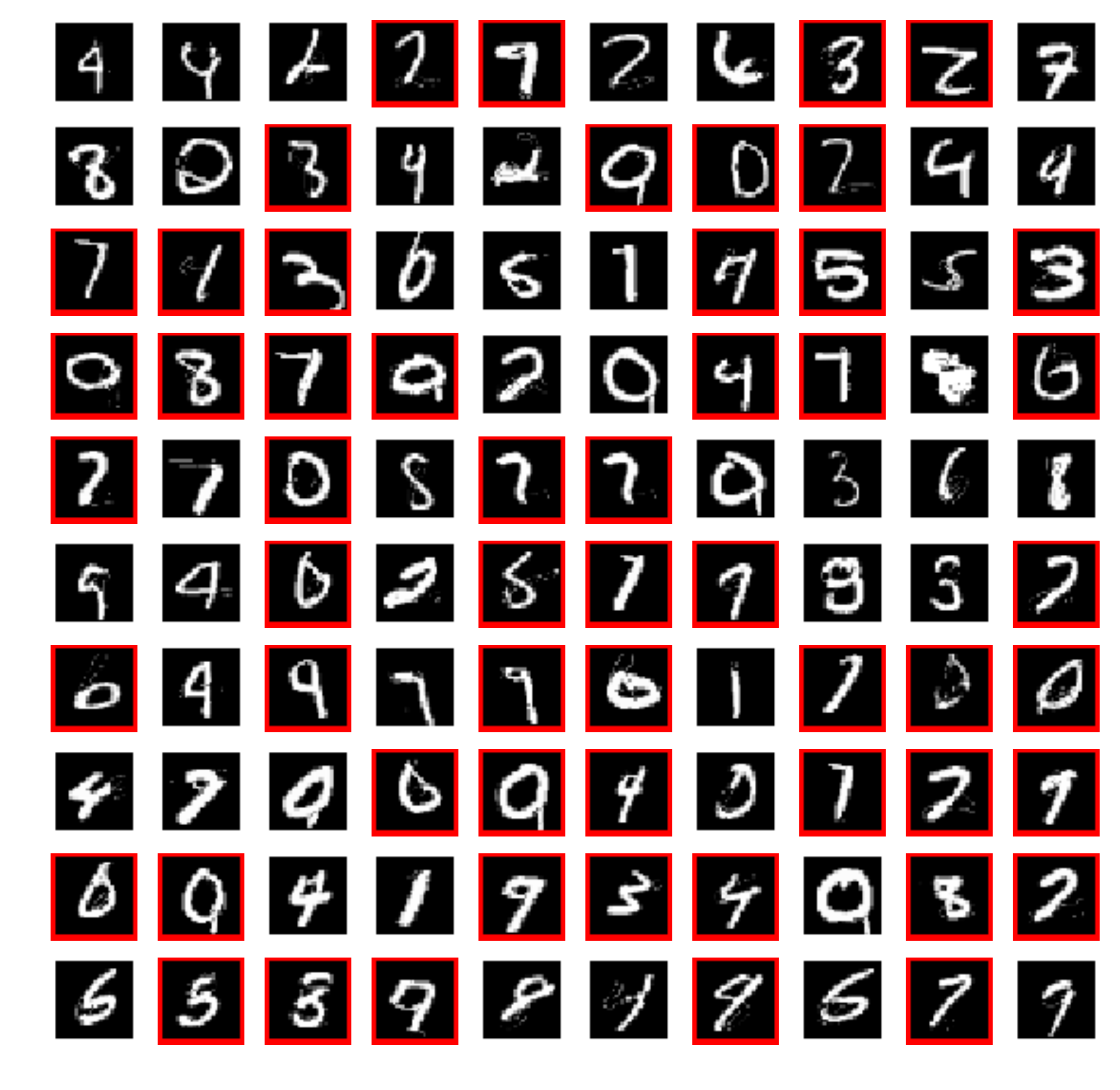}

\subsection{$\ell_\infty$ Invariance Adversarial Examples ($\varepsilon=0.3$)}
\includegraphics[width=\columnwidth]{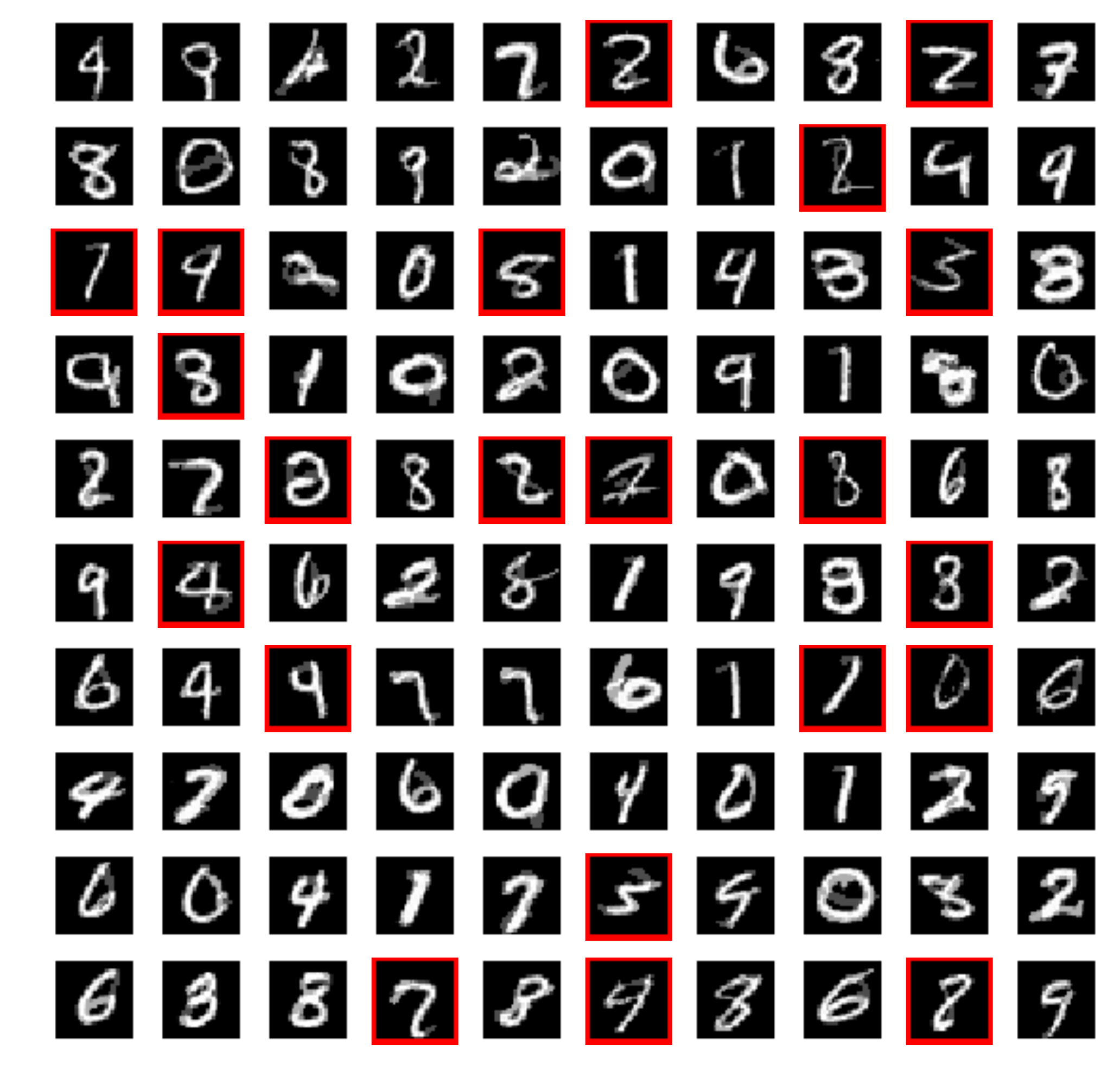}

\subsection{$\ell_\infty$ Invariance Adversarial Examples ($\varepsilon=0.4$)}
\includegraphics[width=\columnwidth]{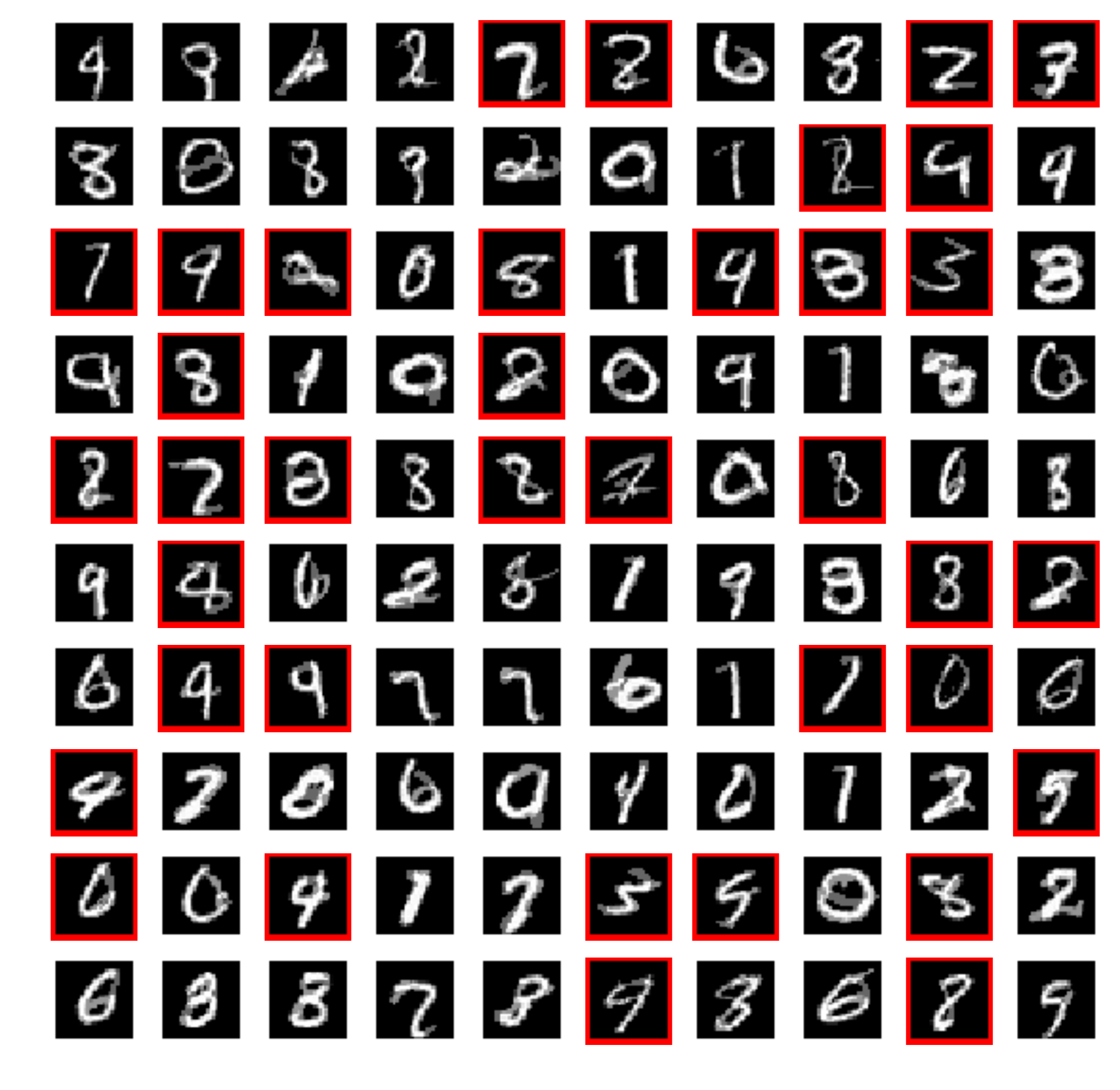}

\begin{figure*}[t]
\centering
\begin{tabular}{@{}l r r r r r r@{}}
    \toprule
    \multicolumn{7}{c}{Agreement between model and the original MNIST label, for sensitivity-based adversarial examples} \\
          \midrule
          \textbf{Model:} & \textbf{Undefended} &   \textbf{$\ell_0$ Sparse}  & \textbf{Binary-ABS} & \textbf{ABS}  & \textbf{$\ell_\infty$ PGD} ($\epsilon=0.3$) & \textbf{$\ell_2$ PGD} ($\epsilon=2$)   \\
         \midrule
         $\ell_0$ Attack ($\epsilon=25$) & 0\%  & 45\% & 63\% & 43\% & 0\% & 40\%\\
         \midrule
         $\ell_\infty$ Attack ($\epsilon=0.3$) & 0\% & 8\% & 77\% & 8\% & 92\% & 1\% \\
         $\ell_\infty$ Attack ($\epsilon=0.4$) & 0\% & 0\% & 60\% & 0\% & 7\% & 0\% \\
         \bottomrule
\end{tabular}
\vspace{-0.25em}
\captionof{table}{Robust model accuracy with respect to the original MNIST labels under different threat models. For $\ell_\infty$ attacks, we use PGD~\cite{madry2017towards}. For $\ell_0$ attacks, we use the PointwiseAttack of~\cite{schott2018towards}.}
\label{tab:modelrobustness}
\end{figure*}

\newpage

\section{Details on Trained Models}
\label{app:hyperparams}

In Section~\ref{sec:experiments}, we evaluate multiple models against invariance adversarial examples. Table~\ref{tab:modelaccuracy} gives results for models taken from prior work. We refer the reader to these works for details. The undefended model is a ResNet-18.

Table~\ref{tab:modelrobustness} reports the standard test accuracy of these models against sensitivity-based adversarial examples. That is, the model is considered correct if it classifiers the adversarial example with the original test-set label of the unperturbed input. To measure $\ell_0$ robustness, we use the PointwiseAttack of~\cite{schott2018towards} repeated $10$ times, with $\epsilon=25$. For $\ell_\infty$ robustness, we use PGD with $100$ iterations for $\epsilon=0.3$ and $\epsilon=0.4$. For the ABS and Binary-ABS models, we report the number from
~\cite{schott2018towards}, for PGD combined with stochastic gradient estimation.

\paragraph{Trading Perturbation-Robustness and Invariance Robustness.}
The adversarially-trained models in Figure~\ref{fig:pgd_invariance} use the same architecture as~\cite{madry2017towards}. We train each model for 10 epochs with Adam and a learning rate of $10^{-3}$ reduced to $10^{-4}$ after 5 epochs (with a batch size of 100). To accelerate convergence, we train against a weaker adversary in the first epoch (with $1/3$ of the perturbation
budget). For training, we use PGD with $40$ iterations for $\ell_\infty$ and $100$ iterations for $\ell_1$. 
For $\ell_\infty$-PGD, we choose a step-size of $2.5 \cdot \varepsilon/k$, where $k$ is the number of attack iterations. For the models trained with $\ell_1$-PGD, we use the Sparse $\ell_1$-Descent Attack of~\citet{tramer2019adversarial}, with a sparsity fraction of $99\%$.

Below, we report the robust accuracy of these models against sensitivity-based adversarial examples, in the sense of~\eqref{eq:robust_err}.

\begin{table}[h]
\centering
\begin{tabular}{@{} l r r r r @{}}
     & \multicolumn{4}{c}{$\epsilon$ for $\ell_\infty$-PGD training} \\
     \cmidrule{2-5}
     \textbf{Attack}& $0.1$ & $0.2$ & $0.3$ & $0.4$ \\
     \toprule
     PGD $\epsilon=0.3$ & 0\% & 6\% & 92\% & 93\% \\
     PGD $\epsilon=0.4$ & 0\% & 0\% & 7\% & 90\% \\
     \bottomrule
\end{tabular}
\caption{Robust model accuracy with respect to the original MNIST label for models trained against $\ell_\infty$ attacks.}
\end{table}

\begin{table}[h]
\centering
\begin{tabular}{@{} l r r r @{}}
     & \multicolumn{3}{c}{$\epsilon$ for $\ell_1$-PGD training} \\
     \cmidrule{2-4}
     \textbf{Attack}& $5$ & $10$ & $15$ \\
     \toprule
     $\ell_0$-PointwiseAttack ($\epsilon=25$) & 41\% & 59\% & 65\% \\
     \bottomrule
\end{tabular}
\caption{Robust model accuracy with respect to the original MNIST label for models trained against $\ell_1$ attacks, and evaluated against $\ell_0$ attacks.}
\end{table}

\paragraph{The Role of Data Augmentation.}
The models in Figure~\ref{fig:RT_eps_models} and Figure~\ref{fig:data_aug} are trained against an adversary that first rotates and translates an input (using the default parameters from~\cite{engstrom2019exploring}) and then adds noise of $\ell_\infty$-norm bounded by $\varepsilon$ to the transformed input. For training, we sample $10$ spatial transformations at random for each input, apply $40$ steps of $\ell_\infty$-PGD to each transformed input, and retain the strongest adversarial example. At test time, we enumerate all possible spatial transformations for each input, and apply $100$ steps of PGD to each.

When training against an adversary with $\varepsilon \geq 0.25$, a warm-start phase is required to ensure training converges. That is, we first trained a model against an $\varepsilon = 0.2$ adversary, and then successively increases $\varepsilon$ by $0.05$ every $5$ epochs.

\section{Proof of Lemma~\ref{lemma:nn}}
\label{apx:proof_nn}

We recall and prove Lemma~\ref{lemma:nn} from Section~\ref{sec:tradeoff}:
\begin{lemma*}
    Constructing an oracle-aligned distance function that satisfies Definition~\ref{def:align} is as hard
    as constructing a function $f$ so that $f(x) = \mathcal{O}(x)$, i.e., $f$ perfectly solves the oracle's classification task.
\end{lemma*}
\begin{proof}
We first show that if we have a distance function $\texttt{dist}$ that satisfies Definition~\ref{def:align}, then the classification task can be perfectly solved.

Let $x$ be an input from class $y$ so that $\mathcal{O}(x) = y$.
Let $\{x_i\}$ be any (possibly infinite) sequence of inputs so that $\texttt{dist}(x, x_i) < \texttt{dist}(x, x_{i+1})$ but so
that $\mathcal{O}(x_i) = y$ for all $x_i$.
Define $l_x = \lim_{i \to \infty} \texttt{dist}(x, x_i)$ as the distance to the furthest input from this class along the path $x_i$.

Assume that $\mathcal{O}$ is not degenerate and there exists at least one input $z$ so that $\mathcal{O}(z) \ne y$. If the problem is degenerate then it is uninteresting: \emph{every} function $\texttt{dist}$ satisfies Definition~3.

Now let $\{z_i\}$ be any (possibly infinite) sequence of inputs so that $\texttt{dist}(x, z_i) > \texttt{dist}(x, z_{i+1})$ and
so that $\mathcal{O}(z_i) \ne y$.
Define $l_z = \lim_{i \to \infty} \texttt{dist}(x, z_i)$ as the distance to the closest input along $z$.
But by Definition~\ref{def:align} we are guaranteed that $l_z > l_x$, otherwise there would exist an index $I$
such that $\texttt{dist}(x, x_I) \geq \texttt{dist}(x, z_I)$ but so that $\mathcal{O}(x) = \mathcal{O}(x_I)$ 
and  $\mathcal{O}(x) \ne \mathcal{O}(z_I)$, contradicting Definition~3.
Therefore for any example $x$, \emph{all} examples $x_i$ that share the same class label
are closer than \emph{any} other input $z$ that has a different class label.

From here it is easy to see that the task can be solved trivially by a 1-nearest neighbor
classifier using this function $\texttt{dist}$.
Let $S = \{(\alpha_i, y_i)\}_{i=1}^C$ contain exactly one pair $(z,y)$ for every class.
Given an arbitrary query point $x$, we can therefore compute the class label as $\text{arg min}\ \texttt{dist}(x, \alpha_i)$,
which must be the correct label, because of the above argument: the closest example from any (incorrect)
class is different than the furthest example from the correct class, and so in particular, the
closest input from $S$ \emph{must} be the correct label.

For the reverse direction, assume we have a classifier $f(x)$ that solves the task perfectly, i.e., $f(x) = \mathcal{O}(x)$ for any $x \in \mathbb{R}^d$. Then the distance function defined as
\[
\texttt{dist}(x, x') = \begin{cases}
0 & \text{ if } f(x) = f(x') \\
1 & \text{otherwise}
\end{cases}
\]
is aligned with the oracle.
\end{proof}

\section{Proofs for the Overly-Robust Features Model}
\label{app:proofs}
We recall the binary classification task from Section~\ref{sec:robust_features}. Unlabeled inputs $x \in \mathbb{R}^{d+2}$ are sampled from some distribution $\mathcal{D}_k^*$ parametrized by $k>1$ as follows:
\begin{gather*}
z \stackrel{\text{u.a.r}}{\sim} \{-1, 1\}, 
\quad x_1 = z/2
\\
x_2 = \begin{cases}
+z & \text{w.p. } \frac{1 + 1/k}{2} \\
-z & \text{w.p. } \frac{1 - 1/k}{2}
\end{cases}
,\ x_3,  \dots, x_{d+2} \stackrel{\text{i.i.d}}{\sim} \mathcal{N}(\frac{z}{\sqrt{d}}, k) \;.
\end{gather*}
The oracle label for an input $x$ is $y = \mathcal{O}(x) = \texttt{sign}(x_1)$.
Note that for $k \gg 1$, features $x_2, \dots, x_{d+2}$ are only weakly correlated with the label $y$.
The oracle labels are robust to $\ell_\infty$-perturbations bounded by $\varepsilon=1/2$:
\begin{claim}
For any $x\sim \mathcal{D}^*$ and $\Delta \in \mathbb{R}^{d+2}$ with $\|\Delta\|_\infty < 1/2$, we have $\mathcal{O}(x) = \mathcal{O}(x+\Delta)$.
\end{claim}
Recall that we consider that a model is trained and evaluated on sanitized and labeled data from this distribution. In this data, the ``noise'' features $x_2, \dots, x_{d+2}$ are more strongly correlated with the oracle labels $y$, and there is a small amount of label noise attributed to mistakes in the data labeling process. 
Specifically, we let $\alpha>0$ and  $\delta> 0$ be small constants, and
define $\mathcal{D}$ as the following distribution: 
\begin{align*}
    x \sim \mathcal{D}^*_{1+\alpha}, 
    \quad y = \begin{cases}
    +\mathcal{O}(x) & \text{w.p. } 1-\delta \\
    -\mathcal{O}(x) & \text{w.p. } \delta
    \end{cases} \;.
\end{align*}

We first show that this sanitization introduces spurious weakly robust features. Standard models trained on $\mathcal{D}$ are thus vulnerable to sensitivity-based adversarial examples. 
\begin{lemma}
Let $f(x)$ be the Bayes optimal classifier on $\mathcal{D}$. Then $f$ agrees with the oracle $\mathcal{O}$ with probability at least $1-\delta$ over $\mathcal{D}$ but with $0\%$ probability against an $\ell_\infty$-adversary bounded by some $\varepsilon = O(d^{-1/2})$.
\end{lemma}
\begin{proof}
The first part of the lemma, namely that $f$ agrees with the oracle $\mathcal{O}$ with probability at least $1-\delta$ follows from the fact that for $(x, y) \sim \mathcal{D}$, $\texttt{sign}(x_1)=y$ with probability $1-\delta$, and $\mathcal{O}(x) = \texttt{sign}(x_1)$. So a classifier that only relies on feature $x_1$ achieves $1-\delta$ accuracy.
To show that the Bayes optimal classifier for $\mathcal{D}$ has adversarial examples, note that this classifier is of the form 
\begin{align*}
f(x) &= \texttt{sign}(w^T x + C) \\
&= \texttt{sign}(w_1 \cdot x_1 + w_2 \cdot x_2 + \sum_{i=3}^{d+2} w_i \cdot x_i + C) \;,
\end{align*}
where $w_1, w_2, C$ are constants, and $w_i = O(1/\sqrt{d})$ for $i \geq 3$. Thus, a perturbation of size $O(1/\sqrt{d})$ applied to features $x_3, \dots, x_{d+2}$ results in a change of size $O(1)$ in $w^T x + C$, which can be made large enough to change the output of $f$ with arbitrarily large probability. As perturbations of size $O(1/\sqrt{d})$ cannot change the oracle's label, they can reduce the agreement between the classifier and oracle to $0\%$.
\end{proof}
Finally, we show that there exists an overly-robust classifier on $\mathcal{D}$ that is vulnerable to invariance adversarial examples:
\begin{lemma}
Let $f(x) = \texttt{sign}(x_2)$. This classifier has accuracy above $1-\alpha/2$ on $\mathcal{D}$, even against an $\ell_\infty$ adversary bounded by $\varepsilon=0.99$. Under such large perturbations, $f$ agrees with the oracle with probability $0\%$.
\end{lemma}
\begin{proof}
The robust accuracy of $f$ follows from the fact that $f(x)$ cannot be changed by any perturbation of $\ell_\infty$ norm strictly below $1$, and that for $(x, y) \sim \mathcal{D}$, we have $x_2 = y$ with probability $\frac{1+1/(1+\alpha)}{2} \geq 1 - \alpha/2$.
For any $(x, y) \sim \mathcal{D}$, note that a perturbation of $\ell_\infty$-norm above $1/2$ can always flip the oracle's label. So we can always find a perturbation $\Delta$ such that $\|\Delta\|_\infty \leq 0.99$ and $f(x+\Delta) \neq \mathcal{O}(x+\Delta)$.
\end{proof}

\end{document}